\documentclass[twoside]{article}

\usepackage[accepted]{aistats2e}
\usepackage{amsthm}
\usepackage{amssymb}
\usepackage{graphicx}
\usepackage{times}
\usepackage{subfigure}
\usepackage{color}
\usepackage{amsbsy}
\usepackage{amsmath}
\usepackage{amsfonts}
\usepackage{epsfig}
\usepackage{wrapfig}

\definecolor{rltred}{rgb}{0.5,0,0}
\definecolor{rltgreen}{rgb}{0,0.5,0}
\definecolor{rltblue}{rgb}{0,0,0.5}

\begin{document}

%

%

\twocolumn[

\aistatstitle{A Last-Step Regression Algorithm for Non-Stationary Online Learning}

\aistatsauthor{ Edward Moroshko \And Koby Crammer }


\aistatsaddress{ Department of Electrical Engineering,\\ The Technion, Haifa, Israel \And Department of Electrical Engineering,\\ The Technion, Haifa, Israel} ]

%

\newtheorem{theorem}{Theorem}
\newtheorem{lemma}[theorem]{Lemma}
\newtheorem{definition}[theorem]{Definition}
\newtheorem{claim}[theorem]{Claim}
\newtheorem{corollary}[theorem]{Corollary}

\def\proofsketch{\par\penalty-1000\vskip .5 pt\noindent{\bf Proof sketch\/: }}
\def\ProofSketch{\par\penalty-1000\vskip .1 pt\noindent{\bf Proof sketch\/: }}
\newcommand{\QED}{\hfill$\;\;\;\rule[0.1mm]{2mm}{2mm}$\\}

\newcommand{\todo}[1]{{~\\\bf TODO: {#1}}~\\}

\newfont{\msym}{msbm10}
\newcommand{\reals}{\mathbb{R}}
\newcommand{\half}{\frac{1}{2}}
\newcommand{\sign}{{\rm sign}}
\newcommand{\paren}[1]{\left({#1}\right)}
\newcommand{\brackets}[1]{\left[{#1}\right]}
\newcommand{\braces}[1]{\left\{{#1}\right\}}
\newcommand{\ceiling}[1]{\left\lceil{#1}\right\rceil}
\newcommand{\abs}[1]{\left\vert{#1}\right\vert}
\newcommand{\tr}{{\rm Tr}}
\newcommand{\pr}[1]{{\rm Pr}\left[{#1}\right]}
\newcommand{\prp}[2]{{\rm Pr}_{#1}\left[{#2}\right]}
\newcommand{\Exp}[1]{{\rm E}\left[{#1}\right]}
\newcommand{\Expp}[2]{{\rm E}_{#1}\left[{#2}\right]}
\newcommand{\eqdef}{\stackrel{\rm def}{=}}
\newcommand{\comdots}{, \ldots ,}
\newcommand{\true}{\texttt{True}}
\newcommand{\false}{\texttt{False}}
\newcommand{\mcal}[1]{{\mathcal{#1}}}
\newcommand{\argmin}[1]{\underset{#1}{\mathrm{argmin}} \:}
\newcommand{\normt}[1]{\left\Vert {#1} \right\Vert^2}
\newcommand{\step}[1]{\left[#1\right]_+}
\newcommand{\1}[1]{[\![{#1}]\!]}
\newcommand{\diag}{{\textrm{diag}}}
\newcommand{\KL}{{\textrm{D}_{\textrm{KL}}}}
\newcommand{\IS}{{\textrm{D}_{\textrm{IS}}}}
\newcommand{\EU}{{\textrm{D}_{\textrm{EU}}}}

\newcommand{\leftmarginpar}[1]{\marginpar[#1]{}}
\newcommand{\figline}{\rule{0.45\textwidth}{0.5pt}}
\newcommand{\figlinee}[1]{\rule{#1\textwidth}{0.5pt}}
\newcommand{\pseudocodefont}{\normalsize}
\newcommand{\nolineskips}{
\setlength{\parskip}{0pt}
\setlength{\parsep}{0pt}
\setlength{\topsep}{0pt}
\setlength{\partopsep}{0pt}
\setlength{\itemsep}{0pt}}

\newcommand{\beq}[1]{\begin{equation}\label{#1}}
\newcommand{\eeq}{\end{equation}}
\newcommand{\beqa}{\begin{eqnarray}}
\newcommand{\eeqa}{\end{eqnarray}}
\newcommand{\secref}[1]{Sec.~\ref{#1}}
\newcommand{\figref}[1]{Fig.~\ref{#1}}
\newcommand{\exmref}[1]{Example~\ref{#1}}
\newcommand{\thmref}[1]{Thm.~\ref{#1}}
\newcommand{\sthmref}[1]{Thm.~\ref{#1}}
\newcommand{\defref}[1]{Definition~\ref{#1}}
\newcommand{\remref}[1]{Remark~\ref{#1}}
\newcommand{\chapref}[1]{Chapter~\ref{#1}}
\newcommand{\appref}[1]{App.~\ref{#1}}
\newcommand{\lemref}[1]{Lem.~\ref{#1}}
\newcommand{\propref}[1]{Proposition~\ref{#1}}
\newcommand{\claimref}[1]{Claim~\ref{#1}}

\newcommand{\corref}[1]{Corollary~\ref{#1}}
\newcommand{\scorref}[1]{Cor.~\ref{#1}}
\newcommand{\tabref}[1]{Table~\ref{#1}}
\newcommand{\tran}[1]{{#1}^{\top}}
\newcommand{\norm}{\mcal{N}}
\newcommand{\eqsref}[1]{Eqns.~(\ref{#1})}

\newcommand{\mb}[1]{{\boldsymbol{#1}}}
\newcommand{\up}[2]{{#1}^{#2}}
\newcommand{\dn}[2]{{#1}_{#2}}
\newcommand{\du}[3]{{#1}_{#2}^{#3}}
\newcommand{\textl}[2]{{$\textrm{#1}_{\textrm{#2}}$}}

\newcommand{\lf}{\lambda_{F}}

\newcommand{\vx}{\mathbf{x}}
\newcommand{\vxi}[1]{\vx_{#1}}
\newcommand{\vxii}{\vxi{t}}

\newcommand{\yi}[1]{y_{#1}}
\newcommand{\yii}{\yi{t}}
\newcommand{\hyi}[1]{\hat{y}_{#1}}
\newcommand{\hyii}{\hyi{i}}

\newcommand{\vy}{\mb{y}}
\newcommand{\vyi}[1]{\vy_{#1}}
\newcommand{\vyii}{\vyi{i}}

\newcommand{\vn}{\mb{\nu}}
\newcommand{\vni}[1]{\vn_{#1}}
\newcommand{\vnii}{\vni{i}}

\newcommand{\vmu}{\mb{\mu}}
\newcommand{\vmus}{{\vmu^*}}
\newcommand{\vmuts}{{\vmus}^{\top}}
\newcommand{\vmui}[1]{\vmu_{#1}}
\newcommand{\vmuii}{\vmui{i}}

\newcommand{\vmut}{\vmu^{\top}}
\newcommand{\vmuti}[1]{\vmut_{#1}}
\newcommand{\vmutii}{\vmuti{i}}

\newcommand{\vsigma}{\mb \sigma}
\newcommand{\msigma}{\Sigma}
\newcommand{\msigmas}{{\msigma^*}}
\newcommand{\msigmai}[1]{\msigma_{#1}}
\newcommand{\msigmaii}{\msigmai{t}}

\newcommand{\mups}{\Upsilon}
\newcommand{\mupss}{{\mups^*}}
\newcommand{\mupsi}[1]{\mups_{#1}}
\newcommand{\mupsii}{\mupsi{i}}
\newcommand{\upssl}{\upsilon^*_l}

\newcommand{\vu}{\mathbf{u}}
\newcommand{\vut}{\tran{\vu}}
\newcommand{\vui}[1]{\vu_{#1}}
\newcommand{\vuti}[1]{\vut_{#1}}
\newcommand{\hvu}{\hat{\vu}}
\newcommand{\hvut}{\tran{\hvu}}
\newcommand{\hvur}[1]{\hvu_{#1}}
\newcommand{\hvutr}[1]{\hvut_{#1}}
\newcommand{\vw}{\mb{w}}
\newcommand{\vwi}[1]{\vw_{#1}}
\newcommand{\vwii}{\vwi{t}}
\newcommand{\vwti}[1]{\vwt_{#1}}
\newcommand{\vwt}{\tran{\vw}}

\newcommand{\tvw}{\tilde{\mb{w}}}
\newcommand{\tvwi}[1]{\tvw_{#1}}
\newcommand{\tvwii}{\tvwi{t}}

\newcommand{\vv}{\mb{v}}
\newcommand{\vvt}{\tran{\vv}}

\newcommand{\vvi}[1]{\vv_{#1}}
\newcommand{\vvti}[1]{\vvt_{#1}}
\newcommand{\lambdai}[1]{\lambda_{#1}}
\newcommand{\Lambdai}[1]{\Lambda_{#1}}

\newcommand{\vxt}{\tran{\vx}}
\newcommand{\hvx}{\hat{\vx}}
\newcommand{\hvxi}[1]{\hvx_{#1}}
\newcommand{\hvxii}{\hvxi{i}}
\newcommand{\hvxt}{\tran{\hvx}}
\newcommand{\hvxti}[1]{\hvxt_{#1}}
\newcommand{\hvxtii}{\hvxti{i}}
\newcommand{\vxti}[1]{\vxt_{#1}}
\newcommand{\vxtii}{\vxti{i}}

\newcommand{\vb}{\mb{b}}
\newcommand{\vbt}{\tran{\vb}}
\newcommand{\vbi}[1]{\vb_{#1}}

\newcommand{\hvy}{\hat{\vy}}
\newcommand{\hvyi}[1]{\hvy_{#1}}


\renewcommand{\mp}{P}
\newcommand{\mpd}{\mp^{(d)}}
\newcommand{\mpt}{\mp^T}
\newcommand{\tmp}{\tilde{\mp}}
\newcommand{\mpi}[1]{\mp_{#1}}
\newcommand{\mpti}[1]{\mpt_{#1}}
\newcommand{\mptii}{\mpti{i}}
\newcommand{\mpii}{\mpi{i}}
\newcommand{\mps}{Q}
\newcommand{\mpsi}[1]{\mps_{#1}}
\newcommand{\mpsii}{\mpsi{i}}
\newcommand{\tmpt}{\tmp^T}
\newcommand{\mz}{Z}
\newcommand{\mv}{V}
\newcommand{\mvi}[1]{\mv_{#1}}
\newcommand{\mvt}{V^T}
\newcommand{\mvti}[1]{\mvt_{#1}}
\newcommand{\mzt}{\mz^T}
\newcommand{\tmz}{\tilde{\mz}}
\newcommand{\tmzt}{\tmz^T}
\newcommand{\mx}{\mathbf{X}}
\newcommand{\ma}{\mathbf{A}}
\newcommand{\Ft}{\mathbf{F}_{t}}
\newcommand{\invFt}{\mathbf{F}_{t}^{-1}}
\newcommand{\FtT}{\mathbf{F}_{t}^{\top}}
\newcommand{\invFtT}{(\FtT)^{-1}}
\newcommand{\mxs}[1]{\mx_{#1}}

\newcommand{\mai}[1]{\ma_{#1}}
\newcommand{\mat}{\tran{\ma}}
\newcommand{\mati}[1]{\mat_{#1}}

\newcommand{\mc}{{C}}
\newcommand{\mci}[1]{\mc_{#1}}
\newcommand{\mcti}[1]{\mct_{#1}}

\newcommand{\md}{{\mathbf{D}}}
\newcommand{\mdi}[1]{\md_{#1}}

\newcommand{\mxi}[1]{\textrm{diag}^2\paren{\vxi{#1}}}
\newcommand{\mxii}{\mxi{i}}

\newcommand{\hmx}{\hat{\mx}}
\newcommand{\hmxi}[1]{\hmx_{#1}}
\newcommand{\hmxii}{\hmxi{i}}
\newcommand{\hmxt}{\hmx^T}
\newcommand{\mxt}{\mx^\top}
\newcommand{\mi}{\mathbf{I}}
\newcommand{\mq}{Q}
\newcommand{\mqt}{\mq^T}
\newcommand{\mlam}{\Lambda}

\renewcommand{\L}{\mcal{L}}
\newcommand{\R}{\mcal{R}}
\newcommand{\X}{\mcal{X}}
\newcommand{\Y}{\mcal{Y}}
\newcommand{\F}{\mcal{F}}
\newcommand{\nur}[1]{\nu_{#1}}
\newcommand{\lambdar}[1]{\lambda_{#1}}
\newcommand{\gammai}[1]{\gamma_{#1}}
\newcommand{\gammaii}{\gammai{i}}
\newcommand{\alphai}[1]{\alpha_{#1}}
\newcommand{\alphaii}{\alphai{i}}
\newcommand{\lossp}[1]{\ell_{#1}}
\newcommand{\eps}{\epsilon}
\newcommand{\epss}{\eps^*}
\newcommand{\lsep}{\lossp{\eps}}
\newcommand{\lseps}{\lossp{\epss}}
\newcommand{\T}{\mcal{T}}

\newcommand{\kc}[1]{\begin{center}\fbox{\parbox{3in}{{\textcolor{green}{KC: #1}}}}\end{center}}
\newcommand{\edward}[1]{\begin{center}\fbox{\parbox{3in}{{\textcolor{red}{EM: #1}}}}\end{center}}
\newcommand{\nv}[1]{\begin{center}\fbox{\parbox{3in}{{\textcolor{blue}{NV: #1}}}}\end{center}}

\newcommand{\newstuffa}[2]{#2}
\newcommand{\newstufffroma}[1]{}
\newcommand{\newstufftoa}{}

\newcommand{\newstuff}[2]{#2}
\newcommand{\newstufffrom}[1]{}
\newcommand{\newstuffto}{}
\newcommand{\oldnote}[2]{}

\newcommand{\commentout}[1]{}
\newcommand{\mypar}[1]{\medskip\noindent{\bf #1}}

\newcommand{\inner}[2]{\left< {#1} , {#2} \right>}
\newcommand{\kernel}[2]{K\left({#1},{#2} \right)}
\newcommand{\tprr}{\tilde{p}_{rr}}
\newcommand{\hxr}{\hat{x}_{r}}
\newcommand{\projalg}{{PST }}
\newcommand{\projealg}[1]{$\textrm{PST}_{#1}~$}
\newcommand{\gradalg}{{GST }}

\newcounter {mySubCounter}
\newcommand {\twocoleqn}[4]{
  \setcounter {mySubCounter}{0} %
  \let\OldTheEquation \theequation %
  \renewcommand {\theequation }{\OldTheEquation \alph {mySubCounter}}%
  \noindent \hfill%
  \begin{minipage}{.40\textwidth}
\vspace{-0.6cm}
    \begin{equation}\refstepcounter{mySubCounter}
      #1
    \end {equation}
  \end {minipage}
~~~~~~
  \addtocounter {equation}{ -1}%
  \begin{minipage}{.40\textwidth}
\vspace{-0.6cm}
    \begin{equation}\refstepcounter{mySubCounter}
      #3
    \end{equation}
  \end{minipage}%
  \let\theequation\OldTheEquation
}

\newcommand{\vzero}{\mb{0}}

\newcommand{\smargin}{\mcal{M}}

\newcommand{\ai}[1]{A_{#1}}
\newcommand{\bi}[1]{B_{#1}}
\newcommand{\aii}{\ai{i}}
\newcommand{\bii}{\bi{i}}
\newcommand{\betai}[1]{\beta_{#1}}
\newcommand{\betaii}{\betai{i}}
\newcommand{\mar}{M}
\newcommand{\mari}[1]{\mar_{#1}}
\newcommand{\marii}{\mari{i}}
\newcommand{\nmari}[1]{m_{#1}}
\newcommand{\nmarii}{\nmari{i}}

\newcommand{\erf}{\Phi}

\newcommand{\var}{V}
\newcommand{\vari}[1]{\var_{#1}}
\newcommand{\varii}{\vari{i}}

\newcommand{\varb}{v}
\newcommand{\varbi}[1]{\varb_{#1}}
\newcommand{\varbii}{\varbi{i}}

\newcommand{\vara}{u}
\newcommand{\varai}[1]{\vara_{#1}}
\newcommand{\varaii}{\varai{i}}

\newcommand{\marb}{m}
\newcommand{\marbi}[1]{\marb_{#1}}
\newcommand{\marbii}{\marbi{i}}

\newcommand{\algname}{{AROW}}
\newcommand{\rlsname}{{RLS}}
\newcommand{\mrlsname}{{MRLS}}

\newcommand{\phia}{\psi}
\newcommand{\phib}{\xi}

\newcommand{\amsigmaii}{\tilde{\msigma}_t}
\newcommand{\amsigmai}[1]{\tilde{\msigma}_{#1}}
\newcommand{\avmuii}{\tilde{\vmu}_i}
\newcommand{\avmui}[1]{\tilde{\vmu}_{#1}}
\newcommand{\amarbii}{\tilde{\marb}_i}
\newcommand{\avarbii}{\tilde{\varb}_i}
\newcommand{\avaraii}{\tilde{\vara}_i}
\newcommand{\aalphaii}{\tilde{\alpha}_i}

\newcommand{\svar}{v}
\newcommand{\smar}{m}
\newcommand{\nsmar}{\bar{m}}

\newcommand{\vnu}{\mb{\nu}}
\newcommand{\vnut}{\vnu^\top}
\newcommand{\vz}{\mb{z}}
\newcommand{\vZ}{\mb{Z}}
\newcommand{\fphi}{f_{\phi}}
\newcommand{\gphi}{g_{\phi}}


\newcommand{\vtmui}[1]{\tilde{\vmu}_{#1}}
\newcommand{\vtmuii}{\vtmui{i}}

\newcommand{\zetai}[1]{\zeta_{#1}}
\newcommand{\zetaii}{\zetai{i}}


\newcommand{\vstate}{\bf{s}}
\newcommand{\vstatet}[1]{\vstate_{#1}}
\newcommand{\vstatett}{\vstatet{t}}

\newcommand{\mtran}{\bf{\Phi}}
\newcommand{\mtrant}[1]{\mtran_{#1}}
\newcommand{\mtrantt}{\mtrant{t}}

\newcommand{\vstatenoise}{\bf{\eta}}
\newcommand{\vstatenoiset}[1]{\vstatenoise_{#1}}
\newcommand{\vstatenoisett}{\vstatenoiset{t}}

\newcommand{\vobser}{\bf{o}}
\newcommand{\vobsert}[1]{\vobser_{#1}}
\newcommand{\vobsertt}{\vobsert{t}}

\newcommand{\mobser}{\bf{H}}
\newcommand{\mobsert}[1]{\mobser_{#1}}
\newcommand{\mobsertt}{\mobsert{t}}

\newcommand{\vobsernoise}{\bf{\nu}}
\newcommand{\vobsernoiset}[1]{\vobsernoise_{#1}}
\newcommand{\vobsernoisett}{\vobsernoiset{t}}

\newcommand{\mstatenoisecov}{\bf{Q}}
\newcommand{\mstatenoisecovt}[1]{\mstatenoisecov_{#1}}
\newcommand{\mstatenoisecovtt}{\mstatenoisecovt{t}}

\newcommand{\mobsernoisecov}{\bf{R}}
\newcommand{\mobsernoisecovt}[1]{\mobsernoisecov_{#1}}
\newcommand{\mobsernoisecovtt}{\mobsernoisecovt{t}}

\newcommand{\vestate}{\bf{\hat{s}}}
\newcommand{\vestatet}[1]{\vestate_{#1}}
\newcommand{\vestatett}{\vestatet{t}}
\newcommand{\vestatept}[1]{\vestatet{#1}^+}
\newcommand{\vestatent}[1]{\vestatet{#1}^-}

\newcommand{\mcovar}{\bf{P}}
\newcommand{\mcovart}[1]{\mcovar_{#1}}
\newcommand{\mcovarpt}[1]{\mcovart{#1}^+}
\newcommand{\mcovarnt}[1]{\mcovart{#1}^-}

\newcommand{\mkalmangain}{\bf{K}}
\newcommand{\mkalmangaint}[1]{\mkalmangain_{#1}}

\newcommand{\vkalmangain}{\bf{\kappa}}
\newcommand{\vkalmangaint}[1]{\vkalmangain_{#1}}

\newcommand{\obsernoise}{{\nu}}
\newcommand{\obsernoiset}[1]{\obsernoise_{#1}}
\newcommand{\obsernoisett}{\obsernoiset{t}}

\newcommand{\obsernoisecov}{r}
\newcommand{\obsernoisecovt}[1]{\obsernoisecov_{#1}}
\newcommand{\obsernoisecovtt}{\obsernoisecov}

\newcommand{\obsnscv}{s}
\newcommand{\obsnscvt}[1]{\obsnscv_{#1}}
\newcommand{\obsnscvtt}{\obsnscvt{t}}

\newcommand{\Psit}[1]{\Psi_{#1}}
\newcommand{\Psitt}{\Psit{t}}

\newcommand{\Omegat}[1]{\Omega_{#1}}
\newcommand{\Omegatt}{\Omegat{t}}

\newcommand{\ellt}[1]{\ell_{#1}}
\newcommand{\gllt}[1]{g_{#1}}

\newcommand{\chit}[1]{\chi_{#1}}

\newcommand{\ms}{\mathcal{M}}
\newcommand{\us}{\mathcal{U}}
\newcommand{\as}{\mathcal{A}}

\newcommand{\mn}{M}
\newcommand{\un}{U}

\newcommand{\seti}[1]{S_{#1}}

\newcommand{\obj}{\mcal{C}}

\newcommand{\dta}[3]{d_{#3}\paren{#1,#2}}

\newcommand{\coa}{a}
\newcommand{\coc}{c}
\newcommand{\cob}{b}
\newcommand{\cor}{r}
\newcommand{\conu}{\nu}

\newcommand{\coat}[1]{\coa_{#1}}
\newcommand{\coct}[1]{\coc_{#1}}
\newcommand{\cobt}[1]{\cob_{#1}}
\newcommand{\cort}[1]{\cor_{#1}}
\newcommand{\conut}[1]{\conu_{#1}}

\newcommand{\coatt}{\coat{t}}
\newcommand{\coctt}{\coct{t}}
\newcommand{\cobtt}{\cobt{t}}
\newcommand{\cortt}{\cort{t}}
\newcommand{\conutt}{\conut{t}}

\newcommand{\rb}{R_B}
\newcommand{\proj}{\textrm{proj}}

\begin{abstract}
  The goal of a learner in standard online learning is to maintain an
  average loss close to the loss of the best-performing {\em single}
  function in some class. In many real-world problems, such as rating
  or ranking items, there is no single best target function during
  the runtime of the algorithm, instead the best (local) target
  function is drifting over time. We develop a novel
  last-step min-max optimal algorithm in
  context of a drift. We analyze the algorithm in the worst-case
  regret framework and show that it maintains an average
  loss close to that of the best slowly changing sequence of linear
  functions, as long as the total of drift is sublinear. In some situations, our bound improves over existing bounds, and additionally the algorithm
  suffers logarithmic regret when there is no drift. 
  We also build on the $H_\infty$ filter and its bound, and develop
  and analyze a second algorithm for drifting setting. Synthetic
  simulations demonstrate the advantages of our algorithms in a
  worst-case constant drift setting.
\end{abstract}

\section{Introduction}
We consider the on-line learning problems, in which a learning
algorithm predicts real numbers given inputs in a sequence of
trials. An example of such a problem is to predict a stock's prices given
input about the current state of the stock-market.  In general, the
goal of the algorithm is to achieve an average loss that is not much
larger compared to the loss one suffers if it had always chosen to
predict according to the best-performing single function from some class
of functions. 

In the past half a century, many algorithms were proposed~(a review
can be found in a comprehensive book on the topic~\cite{CesaBiGa06})
for this problem, some of which are able to achieve an average loss
arbitrarily close to that of the best function in
retrospect. Furthermore, such guarantees hold even if the input and
output pairs are chosen in a fully adversarial manner with no
distributional assumptions.

Competing with the best {\em fixed} function might not suffice for
some problems. In many real-world applications, the true target
function is not fixed, but is slowly drifting over time. Consider a
function designed to rate movies for recommender systems given some
features. Over time a rate of a movie may change as more movies are
released or the season changes. Furthermore, the very own
personal-taste of a user may change as well.

With such properties in mind, we develop new learning algorithms
designed to work with target drift. The goal of an algorithm is to
maintain an average loss close to that of the best slowly changing
sequence of functions, rather than compete well with a single
function. We focus on problems for which this sequence consists only
of linear functions.  Some previous
algorithms~\cite{LittlestoneW94,ECCC-TR00-070,HerbsterW01,KivinenSW01}
designed for this problem are based on gradient descent, with
additional control on the norm (or Bregman divergence) of the
weight-vector used for prediction~\cite{KivinenSW01}, or the number of
inputs used to define it~\cite{CavallantiCG07}.

We take a different route and derive an algorithm based on the
last-step min-max approach proposed by Forster~\cite{Forster} and
later used \cite{TakimotoW00} for online density estimation. On each
iteration the algorithm makes the optimal min-max prediction with
respect to a quantity called regret, assuming it is the last
iteration. Yet, unlike previous work, it is optimal when a drift is
allowed. As opposed to the derivation of the last-step min-max
predictor for a fixed vector, the resulting optimization
problem 
is not straightforward to solve. We develop a dynamic program (a
recursion) to solve this problem, which allows to compute the optimal
last-step min-max predictor. We analyze the algorithm in the
worst-case regret framework and show that the algorithm maintains an
average loss close to that of the best slowly changing sequence of
functions, as long as the total drift is sublinear in the number of
rounds $T$.  Specifically, we show that if the total amount of drift
is $T\nu$ (for $\nu=o(1)$) the cumulative regret is bounded by $T
\nu^{1/3} + \log(T)$. When the instantaneous drift is close to
constant, this improves over a previous bound of Vaits and
Crammer~\cite{VaitsCr11} of an algorithm named ARCOR that showed a
bound of $T \nu^{1/4}\log(T)$. Additionally, when no drift is
introduced (stationary setting) our algorithm suffers logarithmic
regret, as for the algorithm of Forster~\cite{Forster}.  We
also build on the $H_\infty$ adaptive filter, which is min-max optimal
with respect to a {\em filtering task}, and derive another learning
algorithm based on the same min-max principle. We provide a regret
bound for this algorithm as well, and relate the two algorithms and
their respective bounds. Finally, synthetic simulations show the
advantages of our algorithms when a close to constant drift is allowed.

\section{Problem Setting}
\label{sec:problem_setting} 
We focus on the regression task evaluated with the squared loss.  Our
algorithms are designed for the online setting and work in iterations
(or rounds). On each round an online algorithm receives an
input-vector $\vxi{t}\in\reals^d$ and predicts a real value
$\hyi{t}\in\reals$. Then the algorithm receives a target label
$\yi{t}\in\reals$ associated with $\vxi{t}$, uses it to update its
prediction rule, and then proceeds to the next round.

On each round, the performance of the algorithm is evaluated using the
squared loss, $\ell_t(\textrm{alg})=\ell\paren{ \yi{t}, \hyi{t} }
= \paren{\hyi{t}- \yi{t} }^2$. The cumulative loss suffered 
over $T$ iterations is, \(
L_{T}(\textrm{alg})=\sum_{t=1}^{T}\ell_{t}(\textrm{alg})
. 
\) 
The goal of the algorithm is to have low cumulative loss compared to
predictors from some class. A large body of work is focused on linear
prediction functions of the form $f(\vx)=\vxt\vu$ where
$\vu\in\reals^d$ is some weight-vector. We denote by $\ell_t(\vu)
= \paren{\vxti{t}\vu-\yi{t}}^2$ the instantaneous loss of a
weight-vector $\vu$. 

We focus on algorithms that are able to compete against sequences of
weight-vectors, $(\vui{1} \comdots \vui{T})\in \reals^d \times \dots
\times \reals^d$, where $\vui{t}$ is used to make a prediction for the t$th$ example
$(\vxi{t},\yi{t})$.
We define the cumulative loss of such set by
\(
L_T( \{\vui{t}\}) = \sum_t^T \ell_t(\vui{t})
\) and the
  regret of an algorithm by 
\(
R_T(\{\vui{t}\})  = \sum_t^T (\yi{t}- \hyi{t})^2-  L_T(\{\vui{t}\})~.
\)
The goal of the algorithm is to have a low-regret, and formally to have ${R}_T(\{\vui{t}\}) = o(T)$, that is, the
average loss suffered by the algorithm will converge to the average
loss of the best linear function sequence $(\vui{1} \dots \vui{T})$.  

Clearly, with no restriction or penalty over the set $\{\vui{t}\}$ the right term of
the regret can easily be zero by setting, $\vui{t} = \vxi{t}
(\yi{t}/\normt{\vxi{t}})$, which implies $\ell_t(\vui{t})=0$ for all
$t$.
Thus, in the analysis below we incorporate the total drift
of the weight-vectors defined to be,
\begin{align}
V \!\!=\!\! V_T(\{\vui{t}\}) \!\!=\!\! \sum_{t=1}^{T-1} \!\normt{\vui{t}-\vui{t+1}} ~,~\nu \!\!=\!\! \nu(\{\vui{t}\}) \!\!=\!\! \frac{V}{T} ~,
\end{align}
where $\nu$ is the {\em average drift} .
Below we bound the regret with, $R_T(\{\vui{t}\})  \leq
\mcal{O} \paren{ T^\frac{2}{3}V^\frac{1}{3}+\log(T)} =
\mcal{O} \paren{ T \nu^\frac{1}{3}+\log(T)}$.  
Next, we develop an explicit form of the last-step
min-max algorithm with drift. 

\section{Algorithm}
We define the last-step minmax predictor $\hyi{T}$ to be\footnote{$\yi{T}$ and $\hyi{T}$ serve both as quantifiers (over the
 $\min$ and $\max$ operators, respectively), and as the optimal
 arguments of this optimization problem. },
\begin{align}
\arg\min_{\hat{y}_{T}}\max_{y_{T}}&\Bigg[\sum_{t=1}^{T}\left(y_{t}-\hat{y}_{t}\right)^{2}\nonumber\\
&-\min_{\mathbf{u}_{1},\ldots,\mathbf{u}_{T}}Q_{T}\left(\mathbf{u}_{1},\ldots,\mathbf{u}_{T}\right)\Bigg]~,\label{minmax_1}
\end{align}
where we define 
\begin{align}
Q_{t}\left(\mathbf{u}_{1},\ldots,\mathbf{u}_{t}\right) =& b\left\Vert \mathbf{u}_{1}\right\Vert ^{2}+c\sum_{s=1}^{t-1}\left\Vert \mathbf{u}_{s+1}-\mathbf{u}_{s}\right\Vert ^{2}\nonumber\\
& +\sum_{s=1}^{t}\left(y_{s}-\mathbf{u}_{s}^{\top}\mathbf{x}_{s}\right)^{2}~,\label{Q}
\end{align}
for some positive constants $b,c$. 
The last optimization problem can also be seen as a game where the
algorithm chooses a prediction $\hat{y}_t$ to minimize the last-step
regret, while an adversary chooses a target label ${y}_t$ to
maximize it.
The first term of \eqref{minmax_1}
is the loss suffered by the algorithm while
$Q_{t}\left(\mathbf{u}_{1},\ldots,\mathbf{u}_{t}\right)$ defined in
\eqref{Q} is a sum of the loss suffered by some sequence of linear
functions $\left(\mathbf{u}_{1},\ldots,\mathbf{u}_{t}\right)$, a
penalty for consecutive pairs that are far from each other, and for the
norm of the first to be far from zero.

We first solve recursively the inner optimization problem
$\min_{\mathbf{u}_{1},\ldots,\mathbf{u}_{t}}Q_{t}\left(\mathbf{u}_{1},\ldots,\mathbf{u}_{t}\right)$,
for which we
define an auxiliary function,
\begin{align}
P_{t}\left(\mathbf{u}_{t}\right)=\min_{\mathbf{u}_{1},\ldots,\mathbf{u}_{t-1}}Q_{t}\left(\mathbf{u}_{1},\ldots,\mathbf{u}_{t}\right) ~,\label{P}
\end{align}
which clearly satisfies, 
\begin{equation}
\min_{\mathbf{u}_{1},\ldots,\mathbf{u}_{t}}Q_{t}\left(\mathbf{u}_{1},\ldots,\mathbf{u}_{t}\right)
= \min_{\vui{t}} P_t(\vui{t}) ~.\label{PQ}
\end{equation}

We start the derivation of the algorithm with a lemma, stating a recursive form of the function-sequence $P_t(\vui{t})$.
\begin{lemma}
\label{lem:lemma11}
For $t=2,3,\ldots$ 
\begin{align*}
P_1(\mathbf{u}_1)&=Q_1(\mathbf{u}_1)\\ 
 P_{t}\left(\mathbf{u}_{t}\right)&=
\min_{\mathbf{u}_{t-1}}\Bigg(P_{t-1}\left(\mathbf{u}_{t-1}\right)
+c\left\Vert
    \mathbf{u}_{t}-\mathbf{u}_{t-1}\right\Vert
  ^{2} \\
&+\left(y_{t}-\mathbf{u}_{t}^{\top}\mathbf{x}_{t}\right)^{2}\Bigg).\nonumber
\end{align*}
 \end{lemma}
The proof appears in \appref{proof_lemma11}. 
Using \lemref{lem:lemma11} we write explicitly the function $P_t(\vui{t})$. 
\begin{lemma}
\label{lem:lemma12}
The following equality holds 
\begin{align}
P_{t}\left(\mathbf{u}_{t}\right)=\mathbf{u}_{t}^{\top}\mathbf{D}_{t}\mathbf{u}_{t}-2\mathbf{u}_{t}^{\top}\mathbf{e}_{t}+f_{t}~,
\label{eqality_P}
\end{align}
where,
\begin{align}
&\mathbf{D}_{1} \!=b\mathbf{I}+\mathbf{x}_{1}\mathbf{x}_{1}^{\top}
~,~
\mathbf{D}_{t}=\left(\mathbf{D}_{t-1}^{-1}+c^{-1}\mathbf{I}\right)^{-1}+\mathbf{x}_{t}\mathbf{x}_{t}^{\top}\label{D}\\
&\mathbf{e}_{1}\!=y_{1}\mathbf{x}_{1}~,~
\mathbf{e}_{t}=\left(\mathbf{I}+c^{-1}\mathbf{D}_{t-1}\right)^{-1}\mathbf{e}_{t-1}+y_{t}\mathbf{x}_{t}\label{e}\\
&f_{1}\!=y_{1}^{2}~,~
f_{t}=f_{t-1}-\mathbf{e}_{t-1}^{\top}\left(c\mathbf{I}+\mathbf{D}_{t-1}\right)^{-1}\mathbf{e}_{t-1}+y_{t}^{2}\label{f}~.
\end{align}
Note that $\mathbf{D}_{t}\in\mathbb{R}^{d\times d}$ is a positive definite matrix,
$\mathbf{e}_{t}\in\mathbb{R}^{d\times1}$ and $f_{t}\in\mathbb{R}$.
\end{lemma}
The proof appears in \appref{proof_lemma12}.
From \lemref{lem:lemma12} we conclude, by substituting \eqref{eqality_P} in \eqref{PQ}, that, 
\begin{align}
& \min_{\mathbf{u}_{1},\ldots,\mathbf{u}_{t}}Q_{t}\left(\mathbf{u}_{1},\ldots,\mathbf{u}_{t}\right)\nonumber\\
& = \min_{\mathbf{u}_{t}}\left(\mathbf{u}_{t}^{\top}\mathbf{D}_{t}\mathbf{u}_{t}-2\mathbf{u}_{t}^{\top}\mathbf{e}_{t}+f_{t}\right)
= -\mathbf{e}_{t}^{\top}\mathbf{D}_{t}^{-1}\mathbf{e}_{t}+f_{t}
 ~. \label{optimal_Q}
\end{align}
Substituting \eqref{optimal_Q} back in \eqref{minmax_1} we get that the last-step minmax
predictor $\hyi{T}$ is given by,
\begin{eqnarray}
\arg\min_{\hat{y}_{T}}\max_{y_{T}}\left[\sum_{t=1}^{T}\left(y_{t}-\hat{y}_{t}\right)^{2}+\mathbf{e}_{T}^{\top}\mathbf{D}_{T}^{-1}\mathbf{e}_{T}-f_{T}\right]
~. \label{minmax_2}
\end{eqnarray}
Since $\mathbf{e}_{T}$ depends on
$\yi{T} $ we substitute \eqref{e} in the second term
of \eqref{minmax_2}, 
\begin{align}
&\mathbf{e}_{T}^{\top}\mathbf{D}_{T}^{-1}\mathbf{e}_{T}=\nonumber\\
&\left(\left(\mathbf{I}+c^{-1}\mathbf{D}_{T-1}\right)^{-1}\mathbf{e}_{T-1}+y_{T}\mathbf{x}_{T}\right)^{\top}\mathbf{D}_{T}^{-1}\nonumber\\
&\left(\left(\mathbf{I}+c^{-1}\mathbf{D}_{T-1}\right)^{-1}\mathbf{e}_{T-1}+y_{T}\mathbf{x}_{T}\right) ~.\label{second_term}
\end{align}
Substituting \eqref{second_term} and \eqref{f} in \eqref{minmax_2} and
omitting terms not depending explicitly on $y_{T}$ and $\hat{y}_{T}$
we get,
\begin{align}
\hat{y}_T
&= \arg\min_{\hat{y}_{T}}\max_{y_{T}}\bigg[\left(y_{T}-\hat{y}_{T}\right)^{2}  + y_{T}^{2}\mathbf{x}_{T}^{\top}\mathbf{D}_{T}^{-1}\mathbf{x}_{T}\nonumber\\
& \quad
+2y_{T}\mathbf{x}_{T}^{\top}\mathbf{D}_{T}^{-1}\left(\mathbf{I}+c^{-1}\mathbf{D}_{T-1}\right)^{-1}\mathbf{e}_{T-1}-y_{T}^{2}\bigg]\nonumber\\
&=\arg\min_{\hat{y}_{T}}\max_{y_{T}}\bigg[\left(\mathbf{x}_{T}^{\top}\mathbf{D}_{T}^{-1}\mathbf{x}_{T}\right)y_{T}^{2}+\hat{y}_{T}^{2}\label{optimal_y}\\
&\quad+2y_{T}\left(\mathbf{x}_{T}^{\top}\mathbf{D}_{T}^{-1}\left(\mathbf{I}+c^{-1}\mathbf{D}_{T-1}\right)^{-1}\mathbf{e}_{T-1}-\hat{y}_{T}\right)\bigg]~.\nonumber
\end{align}
The last equation is strictly convex in $\yi{T}$ and thus the optimal
solution is not bounded. To solve it, we follow an approach 
used by Forster in a different context~\cite{Forster}. In order to make the optimal
value bounded, we assume that the adversary can only
choose labels from a bounded set $\yi{T}\in[-Y,Y]$. Thus, the optimal
solution of \eqref{optimal_y} over $\yi{T}$ is given by the following
equation, since the optimal value is $\yi{T}\in\{+Y,-Y\}$,
\begin{align*}
\hat{y}_T
&=\arg\min_{\hat{y}_{T}}\bigg[\left(\mathbf{x}_{T}^{\top}\mathbf{D}_{T}^{-1}\mathbf{x}_{T}\right)
  Y^2 +\hat{y}_{T}^{2}\\
& \quad +2 Y\left\vert 
  \mathbf{x}_{T}^{\top}\mathbf{D}_{T}^{-1}\left(\mathbf{I}+c^{-1}\mathbf{D}_{T-1}\right)^{-1}\mathbf{e}_{T-1}-\hat{y}_{T}\right\vert\bigg]~.
\end{align*}
This problem is of a similar form to the one discussed by Forster~\cite{Forster}, from which we get the optimal solution, 
\(
\hyi{T} = clip\paren{
  \mathbf{x}_{T}^{\top}\mathbf{D}_{T}^{-1}\left(\mathbf{I}+c^{-1}\mathbf{D}_{T-1}\right)^{-1}\mathbf{e}_{T-1},Y} ~,
\)
where for $y>0$ we define $clip(x,y)=\sign(x) \min\{\vert x \vert,
y\}$. 
The optimal solution depends explicitly on the bound $Y$, and as its
value is not known, we thus ignore it, and define the output of
the algorithm to be, 
\begin{align}
\hyi{T} =
\mathbf{x}_{T}^{\top}\mathbf{D}_{T}^{-1}\left(\mathbf{I}+c^{-1}\mathbf{D}_{T-1}\right)^{-1}\mathbf{e}_{T-1}
~. \label{my_predictor}
\end{align}
We call the algorithm LASER for last step adaptive regressor
algorithm, and it is summarized in \figref{algorithm:laser}. Clearly, for $c=\infty$ the LASER algorithm reduces to the AAR algorithm of Vovk~\cite{Vovk01}, or the last-step min-max algorithm of Forster~\cite{Forster}. See also the work of Azoury and Warmuth~\cite{AzouryWa01}.
The algorithm can be combined with Mercer kernels as it employs only
sums of inner- and outer-products of its inputs. This algorithm can be
seen also as a forward algorithm~\cite{AzouryWa01}: The predictor of
\eqref{my_predictor} can be seen as the optimal linear {\em model}
obtained over the same prefix of length $T-1$ and the new 
input $\vxi{T}$ with fictional-label $\yi{T}=0$. Specifically, from \eqref{e}
we get that if $\yi{T}=0$, then $\mathbf{e}_{T} = \left(\mathbf{I}+c^{-1}\mathbf{D}_{T-1}\right)^{-1}\mathbf{e}_{T-1}$.
The prediction of the optimal predictor defined in \eqref{optimal_Q} is
$\vxti{T}\mathbf{u}_{T}=\vxti{T}\mathbf{D}_{T}^{-1}\mathbf{e}_{T}=\hat{y}_T$,
where $\hat{y}_T$
was defined in \eqref{my_predictor}.

\section{Analysis}
We now analyze the performance of the algorithm in the worst-case
setting, starting with the following technical lemma.
\begin{lemma}
For all $t$ the following statement holds,
\begin{align*}
&\mathbf{D'}_{t-1}\mathbf{D}_{t}^{-1}\mathbf{x}_{t}\mathbf{x}_{t}^{\top}\mathbf{D}_{t}^{-1}\mathbf{D'}_{t-1}-\mathbf{D}_{t-1}^{-1}\\
&+\mathbf{D'}_{t-1}\left(\mathbf{D}_{t}^{-1}\mathbf{D'}_{t-1}+c^{-1}\mathbf{I}\right)\preceq0
\end{align*}
where $\mathbf{D'}_{t-1}=\left(\mathbf{I}+c^{-1}\mathbf{D}_{t-1}\right)^{-1}$.
\label{lem:technical}
\end{lemma}
The proof appears in \appref{proof_lemma_technical}.
We next bound the cumulative loss of the algorithm,
\begin{theorem}
\label{thm:basic_bound}
Assume the labels are bounded $\sup_t \vert \yi{t} \vert \leq Y$ for some
$Y\in\reals$. Then the following bound holds,

\begin{align*}
L_T(\textrm{LASER})\leq&\min_{\mathbf{u}_{1},\ldots,\mathbf{u}_{T}}\Bigg[b\left\Vert
    \mathbf{u}_{1}\right\Vert ^{2}
+c V_T(\{\vui{t}\})\\
& \quad + L_T(\{\vui{t}\})
\Bigg] +Y^{2}\sum_{t=1}^{T}\mathbf{x}_{t}^{\top}\mathbf{D}_{t}^{-1}\mathbf{x}_{t} ~.
\end{align*}
\end{theorem}
\begin{figure}[t!]
{
\paragraph{Parameters:} $0<b<c$
\paragraph{Initialize:} Set
$\mathbf{D}_{0}=(bc)/(c-b)\,\mathbf{I}\in\reals^{d\times d}$ and $\mathbf{e}_0=\vzero\in\reals^d$\\
{\bf For $t=1 \comdots T$} do
\begin{itemize}
\nolineskips
\item Receive an instance $\vxi{t}$
\item Compute
  $\mathbf{D}_{t}=\left(\mathbf{D}_{t-1}^{-1}+c^{-1}\mathbf{I}\right)^{-1}+\mathbf{x}_{t}\mathbf{x}_{t}^{\top}$
  \hfill\eqref{D}
\item Output  prediction \\ $\hyi{t}=\mathbf{x}_{t}^{\top}\mathbf{D}_{t}^{-1}\left(\mathbf{I}+c^{-1}\mathbf{D}_{t-1}\right)^{-1}\mathbf{e}_{t-1}$
\item Receive the correct label $\yi{t}$
\item
Update:
$\mathbf{e}_{t}=\left(\mathbf{I}+c^{-1}\mathbf{D}_{t-1}\right)^{-1}\mathbf{e}_{t-1}+y_{t}\mathbf{x}_{t}$
\hfill\eqref{e}
\end{itemize}
\paragraph{Output:}  $\mathbf{e}_{T} \ ,\ \mathbf{D}_{T}$\\
}
\figline
\caption{LASER: last step adaptive regression algorithm.}
\label{algorithm:laser}
\end{figure}

\begin{proof}
Fix $t$. A long algebraic manipulation
 yields,
\begin{align}
&\left(y_{t}-\hat{y}_{t}\right)^{2}+\min_{\mathbf{u}_{1},\ldots,\mathbf{u}_{t-1}}Q_{t-1}\left(\mathbf{u}_{1},\ldots,\mathbf{u}_{t-1}\right)\nonumber\\  
&-\min_{\mathbf{u}_{1},\ldots,\mathbf{u}_{t}}Q_{t}\left(\mathbf{u}_{1},\ldots,\mathbf{u}_{t}\right)\nonumber\\
=&\left(y_{t}-\hat{y}_{t}\right)^{2}
+2y_{t}\mathbf{x}_{t}^{\top}\mathbf{D}_{t}^{-1}\mathbf{D'}_{t-1}\mathbf{e}_{t-1}\nonumber\\
&\!+\!\mathbf{e}_{t-1}^{\top}\!\Bigg[ \!-\!\mathbf{D}_{t-1}^{-1} 
\!+\!\!
\mathbf{D'}_{t-1}
\!\left(\mathbf{D}_{t}^{-1}
\mathbf{D'}_{t-1}
\!\!
\!+\!c^{-1}\mathbf{I}\right)\!\!\Bigg]\!\mathbf{e}_{t-1}\nonumber\\
&+y_{t}^{2}\mathbf{x}_{t}^{\top}\mathbf{D}_{t}^{-1}\mathbf{x}_{t}-y_{t}^{2}~.
\label{step1}
\end{align} 
Substituting the specific value of the predictor 
$\hat{y}_{t}=\mathbf{x}_{t}^{\top}\mathbf{D}_{t}^{-1}\mathbf{D'}_{t-1}\mathbf{e}_{t-1}$
from \eqref{my_predictor}, we get that \eqref{step1} equals to,
\begin{align}
  & \hat{y}_{t}^{2}+y_{t}^{2}\mathbf{x}_{t}^{\top}\mathbf{D}_{t}^{-1}\mathbf{x}_{t}+\mathbf{e}_{t-1}^{\top}\Bigg[-\mathbf{D}_{t-1}^{-1}\nonumber\\
&+\mathbf{D'}_{t-1}\left(\mathbf{D}_{t}^{-1}\mathbf{D'}_{t-1}+c^{-1}\mathbf{I}\right)\Bigg]\mathbf{e}_{t-1}\nonumber\\
 =& \mathbf{e}_{t-1}^{\top}\mathbf{D'}_{t-1}\mathbf{D}_{t}^{-1}\mathbf{x}_{t}\mathbf{x}_{t}^{\top}\mathbf{D}_{t}^{-1}\mathbf{D'}_{t-1}\mathbf{e}_{t-1}+\mathbf{e}_{t-1}^{\top}\Bigg[-\mathbf{D}_{t-1}^{-1}\nonumber\\
&+\mathbf{D'}_{t-1}\left(\mathbf{D}_{t}^{-1}\mathbf{D'}_{t-1}+c^{-1}\mathbf{I}\right)\Bigg]\mathbf{e}_{t-1}
 +y_{t}^{2}\mathbf{x}_{t}^{\top}\mathbf{D}_{t}^{-1}\mathbf{x}_{t}\nonumber\\
 =& \mathbf{e}_{t-1}^{\top}
\Bigg[\mathbf{D'}_{t-1}\mathbf{D}_{t}^{-1}\mathbf{x}_{t}\mathbf{x}_{t}^{\top}\mathbf{D}_{t}^{-1}\mathbf{D'}_{t-1}-\mathbf{D}_{t-1}^{-1}\label{step2}\\
&+\mathbf{D'}_{t-1}\left(\mathbf{D}_{t}^{-1}\mathbf{D'}_{t-1}+c^{-1}\mathbf{I}\right)\Bigg]\mathbf{e}_{t-1}+y_{t}^{2}\mathbf{x}_{t}^{\top}\mathbf{D}_{t}^{-1}\mathbf{x}_{t} ~.\nonumber
\end{align}
Using \lemref{lem:technical} we upper bound
\eqref{step2} with,
\(
y_{t}^{2}\mathbf{x}_{t}^{\top}\mathbf{D}_{t}^{-1}\mathbf{x}_{t} \leq Y^{2}\mathbf{x}_{t}^{\top}\mathbf{D}_{t}^{-1}\mathbf{x}_{t}
\) .
Finally, summing over $t\in\left\{ 1,\ldots,T\right\} $ gives the
desired bound,
\begin{align*}
&\sum_{t=1}^{T}\left(y_{t}-\hat{y}_{t}\right)^{2}-\min_{\mathbf{u}_{1},\ldots,\mathbf{u}_{T}} \left[b\left\Vert
    \mathbf{u}_{1}\right\Vert ^{2} +c\sum_{t=1}^{T-1}\left\Vert
    \mathbf{u}_{t+1}-\mathbf{u}_{t}\right\Vert
  ^{2}\right.\\
  & \left. \quad+\sum_{t=1}^{T}\left(y_{t}-\mathbf{u}_{t}^{\top}\mathbf{x}_{t}\right)^{2}\right]\\
=~ &L_T(\textrm{LASER})\!\!-\!\!\min_{\mathbf{u}_{1},\ldots,\mathbf{u}_{T}}\Bigg[b\left\Vert
    \mathbf{u}_{1}\right\Vert ^{2}
\!+\!c V_T(\{\vui{t}\})
+ L_T(\{\vui{t}\})
\Bigg]\\
\leq~&
Y^{2}\sum_{t=1}^{T}\mathbf{x}_{t}^{\top}\mathbf{D}_{t}^{-1}\mathbf{x}_{t}
\end{align*}
\end{proof}
\begin{figure}[t!]
{
\paragraph{Parameters:} $1<a$ ~,~ $ 0<b,c$
\paragraph{Initialize:} Set
$\mathbf{P}_{0}=b^{-1}\mi\in\reals^{d\times d}$ and $\mathbf{w}_0=\vzero\in\reals^d$\\
{\bf For $t=1 \comdots T$} do
\begin{itemize}
\nolineskips
\item Receive an instance $\vxi{t}$
\item Output  prediction $\hyi{t}=\mathbf{x}_{t}^{\top}\mathbf{w}_{t-1}$
\item Receive the correct label $\yi{t}$
\item Compute $ \tilde{\mathbf{P}}_{t}=\left( \mathbf{P}_{t-1}^{-1}  + (a-1)\vxi{t}\vxti{t}\right)^{-1}$
\item
Update $\mathbf{w}_{t}=\mathbf{w}_{t-1} + a\tilde{\mathbf{P}}_{t} 
(\yi{t}-\hyi{t}) \vxi{t}$ 
\item Update $\mathbf{P}_t = \tilde{\mathbf{P}}_t + c^{-1}\mi$
\end{itemize}
\paragraph{Output:}  $\mathbf{w}_{T} \ ,\ \mathbf{P}_{T}$\\
}
\figline
\caption{An $H_\infty$ algorithm for online regression.}
\label{algorithm:hi}
\end{figure}

In the next lemma we further bound the right term of
\thmref{thm:basic_bound}.
This type of bound is based on the
usage of the covariance-like matrix $\mathbf{D}$.
\begin{lemma}
\label{lem:bound_1}
\begin{equation}
\sum_{t=1}^{T}\mathbf{x}_{t}^{\top}\mathbf{D}_{t}^{-1}\mathbf{x}_{t}\leq\ln\left|\frac{1}{b}\mathbf{D}_{T}\right|
+c^{-1}\sum_{t=1}^{T} \tr\paren{\mathbf{D}_{t-1}} ~.\label{covariance_bound}
\end{equation}
\end{lemma}
\begin{proof}
Similar to the derivation of Forster~\cite{Forster} (details omitted due
to lack of space),
\begin{align*}
\mathbf{x}_{t}^{\top}\mathbf{D}_{t}^{-1}\mathbf{x}_{t}
& \leq  \ln\frac{\left|\mathbf{D}_{t}\right|}{\left|\mathbf{D}_{t}-\mathbf{x}_{t}\mathbf{x}_{t}^{\top}\right|}
  =  \ln\frac{\left|\mathbf{D}_{t}\right|}{\left|\left(\mathbf{D}_{t-1}^{-1}+c^{-1}\mathbf{I}\right)^{-1}\right|}\\
 & =  \ln\frac{\left|\mathbf{D}_{t}\right|}{\left|\mathbf{D}_{t-1}\right|}\left|\left(\mathbf{I}+c^{-1}\mathbf{D}_{t-1}\right)\right|\\
 & =
 \ln\frac{\left|\mathbf{D}_{t}\right|}{\left|\mathbf{D}_{t-1}\right|}+\ln\left|\left(\mathbf{I}+c^{-1}\mathbf{D}_{t-1}\right)\right| ~.
\end{align*}
 and because $\ln\left|\frac{1}{b}\mathbf{D}_{0}\right| \geq 0$ we get
\(
\sum_{t=1}^{T}\mathbf{x}_{t}^{\top}\mathbf{D}_{t}^{-1}\mathbf{x}_{t}\leq
\ln\left|\frac{1}{b}\mathbf{D}_{T}\right|
+\sum_{t=1}^{T}\ln\left|\left(\mathbf{I}+c^{-1}\mathbf{D}_{t-1}\right)\right|
\leq \ln\left|\frac{1}{b}\mathbf{D}_{T}\right|
 +c^{-1}\sum_{t=1}^{T} \tr\paren{\mathbf{D}_{t-1}} ~.
\)
\end{proof}

At first sight it seems that the right term of
\eqref{covariance_bound} may grow super-linearly with $T$, as each of the
matrices $\mdi{t}$ grows with $t$. The next two lemmas show that this
is not the case, and in fact, the right term of
\eqref{covariance_bound} is not growing too fast, which will allow us to
obtain a sub-linear regret bound. \lemref{operator_scalar} analyzes
the properties of the recursion of $\mathbf{D}$ defined in
\eqref{D} for scalars, that is $d=1$. In \lemref{eigen_values_lemma} we extend
this analysis to matrices.
\begin{lemma}
\label{operator_scalar}
Define 
\(
f(\lambda) = {\lambda \beta}/\paren{\lambda+ \beta} + x^2 
\)
for $\beta,\lambda \geq 0$ and some $x^2 \leq \gamma^2$. Then: {\bf (1)} $f(\lambda) \leq \beta +
\gamma^2$ {\bf (2)} $f(\lambda) \leq \lambda + \gamma^2$ {\bf (3)} $f(\lambda) \leq \max\braces{\lambda,\frac{3\gamma^2 +
  \sqrt{\gamma^4+4\gamma^2\beta}}{2}}$ ~.
\end{lemma}
The proof appears in \appref{proof_operator_scalar}.
We build on \lemref{operator_scalar} to bound the maximal eigenvalue of the
matrices $\mdi{t}$.
\begin{lemma}
\label{eigen_values_lemma}
Assume $\normt{\vxi{t}} \leq X^2$ for some $X$. Then, the eigenvalues
of $\mdi{t}$ (for $t \geq 1$), denoted by $\lambda_i\paren{\mdi{t}}$, are upper bounded by $\max_i\lambda_i\paren{\mdi{t}}\leq\max\braces{ \frac{3X^2 +
  \sqrt{X^4+4X^2 c}}{2},b+X^2} $.
\end{lemma}
\begin{proof}
  By induction.  From \eqref{D} we have that
  $\lambda_i(\mathbf{D}_{1}) \leq b + X^2$  for
$i=1\comdots d$. We proceed with a proof for some $t$. For simplicity,
denote by $\lambda_i = \lambda_i(\mdi{t-1})$ the i$th$
eigenvalue of $\mdi{t-1}$ with a corresponding
eigenvector $\vvi{i}$. 
From 
\eqref{D} we have,
\begin{align}
\mathbf{D}_{t}
&=\left(\mathbf{D}_{t-1}^{-1}+c^{-1}\mathbf{I}\right)^{-1}+\mathbf{x}_{t}\mathbf{x}_{t}^{\top}\nonumber\\
& \preceq \left(\mathbf{D}_{t-1}^{-1}+c^{-1}\mathbf{I}\right)^{-1}  +
\mi \normt{\vxi{t}}\nonumber\\
 &= \sum_i^d \vvi{i}
\vvti{i}\paren{ \frac{\lambda_i c }{\lambda_i + c} +
  \normt{\vxi{t}}} ~.\label{bound_eigens}
\end{align}
Plugging  \lemref{operator_scalar} in \eqref{bound_eigens} we get, 
\(
\mathbf{D}_{t} 
\preceq \sum_i^d \vvi{i}
\vvti{i}\max\braces{ \frac{3X^2 +
  \sqrt{X^4+4X^2 c}}{2},b+X^2} 
= \max\braces{ \frac{3X^2 +
  \sqrt{X^4+4X^2 c}}{2},b+X^2} \mi~.
\)
\end{proof}

Finally, equipped with the above lemmas we 
prove the main
result of this section.
\begin{corollary}
\label{cor:main}
Assume $\normt{\vxi{t}}\leq X^2$, $\vert\yi{t}\vert \leq Y$. Then, 
\begin{align}
L_T(\textrm{LASER})\leq 
  b\left\Vert
     \mathbf{u}_{1}\right\Vert ^{2}+L_T(\{\vui{t}\})+Y^{2} \ln\left|\frac{1}{b}\mathbf{D}_{T}\right|\nonumber\\ +c^{-1}Y^2\tr\paren{\mathbf{D}_0}+c V\nonumber\\
 +c^{-1}Y^2 T d  \max\braces{ \frac{3X^2 +
   \sqrt{X^4+4X^2 c}}{2},b+X^2}~.\label{final_cor}
\end{align} 

Furthermore, set
$b=\varepsilon c$ for some $0<\varepsilon<1$.
Denote by 
\(
\mu = \max\braces{9/8X^2, \frac{\paren{b+X^2}^2}{8X^2}}
\) and 
\(M =
\max\braces{3X^2, b+X^2}
\).
If $V \leq T \frac{\sqrt{2}Y^2dX}{\mu^{3/2}}$ (low drift) then by setting 
\begin{align}
c= \paren{{\sqrt{2}T Y^2 d X}/{V}}^{2/3}\label{c1}
\end{align}
 we have,
\begin{align}
& L_T(\textrm{LASER})
\leq\nonumber\\ 
& \quad b\left\Vert \mathbf{u}_{1}\right\Vert ^{2} + 3\paren{\sqrt{2}Y^2 d X}^{2/3} T^{2/3} V^{1/3}\nonumber\\
&\quad +\frac{\varepsilon}{1-\varepsilon}Y^{2}d +L_T(\{\vui{t}\})
+Y^{2} \ln\left|\frac{1}{b}\mathbf{D}_{T}\right|\label{bound1}~.
\end{align}
\end{corollary}
The proof appears in \secref{proof_cor_main}.
A few remarks are in order. First, when the total drift $V=0$ goes to
zero, we set $c=\infty$ and thus we have $\mdi{t}=b\mi+\sum_{s=1}^t
\vxi{s}\vxti{s}$ used in recent
algorithms~\cite{Vovk01,Forster,Hayes,CesaBianchiCoGe05}. In this case
the algorithm reduces to the algorithm by Forster~\cite{Forster} (which is also the Aggregating Algorithm for Regression of Vovk~\cite{Vovk01}), with
the same logarithmic regret bound (note that the last term of \eqref{bound1} is logarithmic in $T$, see the proof of
Forster~\cite{Forster}). See also the work of Azoury and Warmuth~\cite{AzouryWa01}.
Second, substituting $V=T\nu$ we get that the
bound depends on the average drift as $T^{2/3}
(T\nu)^{1/3}=T\nu^{1/3}$. Clearly, to have a sublinear regret we must
have $\nu=o(1)$.
Third, Vaits and Crammer~\cite{VaitsCr11} recently proposed an
algorithm, called ARCOR, for the same setting. The regret of ARCOR
depends on the total drift as $\sqrt{T V'}\log(T)$, where their
definition of total drift is a sum of the Euclidean differences
$V'=\sum_t^{T-1} \Vert \vui{t+1}-\vui{t}\Vert$, rather than the
squared norm. When the instantaneous drift
$\Vert\vui{t+1}-\vui{t}\Vert$ is constant, this notion of total drift
is related to our average drift, $V'=T\sqrt{\nu}$. Therefore, in this
case the bound of ARCOR~\cite{VaitsCr11} is $\nu^{1/4} T \log(T)$
which is worse than our bound, both since it has an additional
$\log(T)$ factor (as opposed to our additive log term) and since
$\nu=o(1)$. Therefore we expect that our algorithm will perform better
than ARCOR~\cite{VaitsCr11} when the instantaneous drift is
approximately constant. Indeed, the synthetic simulations described in
\secref{sec:simulations} further support this conclusion. Fourth,
Herbster and Warmuth \cite{HerbsterW01} developed shifting bounds for
general gradient descent algorithms with projection of the
weight-vector using the Bregman divergence. In their bounds, there is
a factor greater than 1 multiplying the term
$L_T\paren{\braces{\vui{t}}}$, leading to a small regret only when
the data is close to be realizable with linear models. Yet, their
bounds have better dependency on $d$, the dimension of the inputs $x$.
Busuttil and Kalnishkan~\cite{BusuttilK07} developed a variant of the
Aggregating Algorithm~\cite{vovkAS} for the non-stationary
setting. However, to have sublinear regret they require a strong
assumption on the drift $V=o(1)$, while we require only
$V=o(T)$. Fifth, if $V \geq T \frac{Y^2dM}{\mu^{2}}$ then by setting 
\(
c=
 \sqrt{{Y^2dMT}/{V}}
\) 
 we have,
\begin{align}
& L_T(\textrm{LASER})
\leq b\left\Vert \mathbf{u}_{1}\right\Vert ^{2}
 + 2\sqrt{Y^2 d TMV}\nonumber\\
& \quad+\frac{\varepsilon}{1-\varepsilon}Y^{2}d +L_T(\{\vui{t}\})
+Y^{2} \ln\left|\frac{1}{b}\mathbf{D}_{T}\right| \label{high_drift}
\end{align}
(See \appref{details_for_second_bound} for details). The last bound is linear in $T$ and can be obtained also by a naive
algorithm that outputs $\hat{y}_t=0$ for all $t$.

\section{An $H_\infty$ Algorithm for Online Regression}
\label{H8_sec}
Adaptive filtering is an active and well established area of research
in signal processing. Formally, it is equivalent to online learning. On
each iteration $t$ the filter receives an input $\vxi{t}\in\reals^d$ and
predicts a corresponding output $\hyi{t}$. It then receives the true
desired output $\yi{t}$ and updates its internal model. Many adaptive
filtering algorithms employ linear models, that is, at time $t$ they
output $\hyi{t} = \vwti{t}\vxi{t}$. For
example, a well known online learning algorithm~\cite{WidrowHoff}
for regression, which is basically a gradient-descent algorithm with
the squared-loss, is known as the {\em least mean-square (LMS)}
algorithm in the adaptive filtering
literature~\cite{Sayed:2008:AF:1370975}.

One possible difference between adaptive filtering and online learning
can be viewed in the interpretation of algorithms, and as a
consequence, of their analysis. In online learning, the goal of an
algorithm is to make {\em predictions} $\hyi{t}$, and the predictions are
compared to the predictions of some function from a
known class (e.g. linear, parameteized
by $\vu$). Thus, a typical online performance bound relates the quality of the
algorithm's predictions with the quality of some function's $g(\vx)=\vut\vx$
predictions, using some non-negative loss measure $\ell(\vwti{t}\vxi{t},\yi{t})$. Such
bounds often have the following shape,
\[
\overbrace{\sum_t \ell(\vwti{t}\vxi{t},\yi{t})}^{\textrm{algorithm
    loss with respect to observation}} \leq
A \overbrace{\sum_t \ell( \vut\vxi{t}, \yi{t})}^{\textrm{function $\vu$ loss}} + B,
\]
for some multiplicative-factor $A$ and an additive factor $B$. 

Adaptive filtering is similar to the realizable setting in machine
learning, where it is assumed the existence of some filter and
the goal is to recover it using {\em noisy} observations. Often it is assumed that the output is a corrupted version of the output
of some function, $y=f(\vx)+n$, with some noise $n$. Thus a
typical bound relates the quality of an algorithm's predictions {\em
  with respect to the target filter} $\vu$ and the amount of noise in the problem,
\[
\overbrace{\sum_t
  \ell(\vwti{t}\vxi{t},\vut\vxi{t})}^{\textrm{algorithm loss with
    respect to a reference}} \leq
A \overbrace{\sum_t \ell( \vut\vxi{t}, \yi{t})}^{\textrm{amount of
    noise}} + B ~.
\]

The $H_\infty$ filters~(see e.g. papers by Simon~\cite{Simon:2006:OSE:1146304,DBLP:journals/tsp/Simon06})
 are a family of (robust) linear filters developed based on a min-max approach, like LASER, and analyzed in the worst
case setting. These filters are reminiscent of the celebrated Kalman
filter~\cite{Kalman60}, which was motivated and analyzed in a stochastic
setting with Gaussian noise. A pseudocode of one such filter we {\em modified }to online linear
regression appears in \figref{algorithm:hi}.
Theory of $H_\infty$ filters states~\cite[Section 11.3]{Simon:2006:OSE:1146304} the
following bound on its performance as a filter.
\begin{theorem}
Assume the filter is executed with parameters $a>1$ and $b,c>0$.
Then, for all input-output pairs $(\vxi{t},\yi{t})$ and for all
reference vectors $\vui{t}$ the following bound holds on the filter's performance,
\(
\sum_{t=1}^{T}\left(\mathbf{x}_{t}^{\top}\mathbf{w}_{t}-\vxti{t}\mathbf{u}_{t}\right)^{2}
\leq
a L_T(\{\vui{t}\}) 
+b\left\Vert
  \mathbf{u}_{1}\right\Vert
^{2}+c
V_T\paren{\braces{\vui{t}}}
~.
\)
\end{theorem}
From the theorem we establish a regret bound for the $H_\infty$
algorithm to online learning.
\begin{corollary}
Fix $\alpha>0$.
The total squared-loss suffered by the algorithm  
is bounded by
\begin{eqnarray}
L_T(H_\infty)&\leq&\left(1+{1}/{\alpha}+\left(1+\alpha\right)a\right) L_T(\braces{\vui{t}})
\label{bound_h8}\\
&&+\left(1+\alpha\right)b\left\Vert
   \mathbf{u}_{1}\right\Vert ^{2}+\left(1+\alpha\right)c
V_T\paren{\braces{\vui{t}}} ~.
\nonumber
\end{eqnarray}

\end{corollary}
\begin{proof}
Using a bound of Hassibi and Kailath~\cite[Lemma 4]{HassibiKa97} we have that for all $\alpha>0$,
\(
\left(y_{t}-\mathbf{x}_{t}^{\top}\mathbf{w}_{t}\right)^{2}\leq\left(1+\frac{1}{\alpha}\right)\left(y_{t}-\mathbf{x}_{t}^{\top}\mathbf{u}_{t}\right)^{2}+\left(1+\alpha\right)\left[\mathbf{x}_{t}^{\top}\left(\mathbf{w}_{t}-\mathbf{u}_{t}\right)\right]^{2}
\). 
Plugging back into the theorem and collecting the terms we get the
desired bound.
\end{proof}
The bound holds for any $\alpha>0$. We plug 
$\alpha = \sqrt{{L_T\paren{\braces{\vui{t}}}}/\paren{a L_T\paren{\braces{\vui{t}}} + cV + b \normt{\vui{1}}}}$
%
in \eqref{bound_h8} to get,
\begin{align*}
L_T(H_\infty) 
\leq&~ (1+a)  L_T\paren{\braces{\vui{t}}} + cV + b
\normt{\vui{1}} \\
&+ 2 \sqrt{\paren{a L_T\paren{\braces{\vui{t}}} + cV + b
    \normt{\vui{1}}}{L_T\paren{\braces{\vui{t}}}}}\\
\leq&~ (1+a+2\sqrt{a})  L_T\paren{\braces{\vui{t}}} + cV + b
\normt{\vui{1}} \\
&+ 2 \sqrt{\paren{cV + b
    \normt{\vui{1}}}{L_T\paren{\braces{\vui{t}}}}} ~.
\end{align*}
Intuitively, we expect the $H_\infty$ algorithm to perform better when
the data is close to linear, that is when
$L_T\paren{\braces{\vui{t}}}$ is small, as, conceptually, it was
designed to minimize a loss with respect to weights $\{\vui{t}\}$.  On
the other hand, LASER is expected to perform better when the data is
hard to predict with linear models, as it is not motivated from this
assumption. Indeed, the bounds reflect these observations.

Comparing the last bound with \eqref{bound1} we note a few
differences. First, the factor $\paren{1+a+2\sqrt{a}}\geq4$ of
$L_T\paren{\braces{\vui{t}}}$ is worse for $H_\infty$ than for
LASER, which is a unit. Second, LASER has worse
dependency in the drift $T^{2/3}V^{1/3}$, while  for $H_\infty$ it is
about $cV + 2 \sqrt{{cV }{L_T\paren{\braces{\vui{t}}}}}$. Third, the
$H_\infty$ has an additive factor $\sim
\sqrt{L_T\paren{\braces{\vui{t}}}}$, while LASER has an
additive 
logarithmic factor, at most.

Hence, the bound of the $H_\infty$ based algorithm is
better when the cumulative loss $L_T\paren{\braces{\vui{t}}}$ is
small. In this case, $4 L_T\paren{\braces{\vui{t}}}$ is not a large
quantity, and as all the other quantities behave like
$\sqrt{L_T\paren{\braces{\vui{t}}}}$, they are small as well. On the other
hand, if $L_T\paren{\braces{\vui{t}}}$ is large, and is linear in $T$, the first term of the
bound becomes dominant, and thus the factor of $4$ for the $H_\infty$
algorithm makes its bound higher than that of LASER. Both bounds were obtained from a min-max approach, either directly (LASER) or
via-reduction from filtering ($H_\infty$). The bound of the former is
lower in hard problems.  Kivinen et al.~\cite{KivinenWaHa03} proposed
 another approach for filtering with a bound depending on $\sum_t
 \Vert \vui{t} \!-\! \vui{t-1} \Vert$ and not the sum of squares as we
 have both for LASER and the $H_\infty$-based algorithm.

\section{Simulations}
\vspace{-0.15cm}
\label{sec:simulations}

\begin{figure}[t!]
\subfigure{\includegraphics[width=0.24\textwidth]{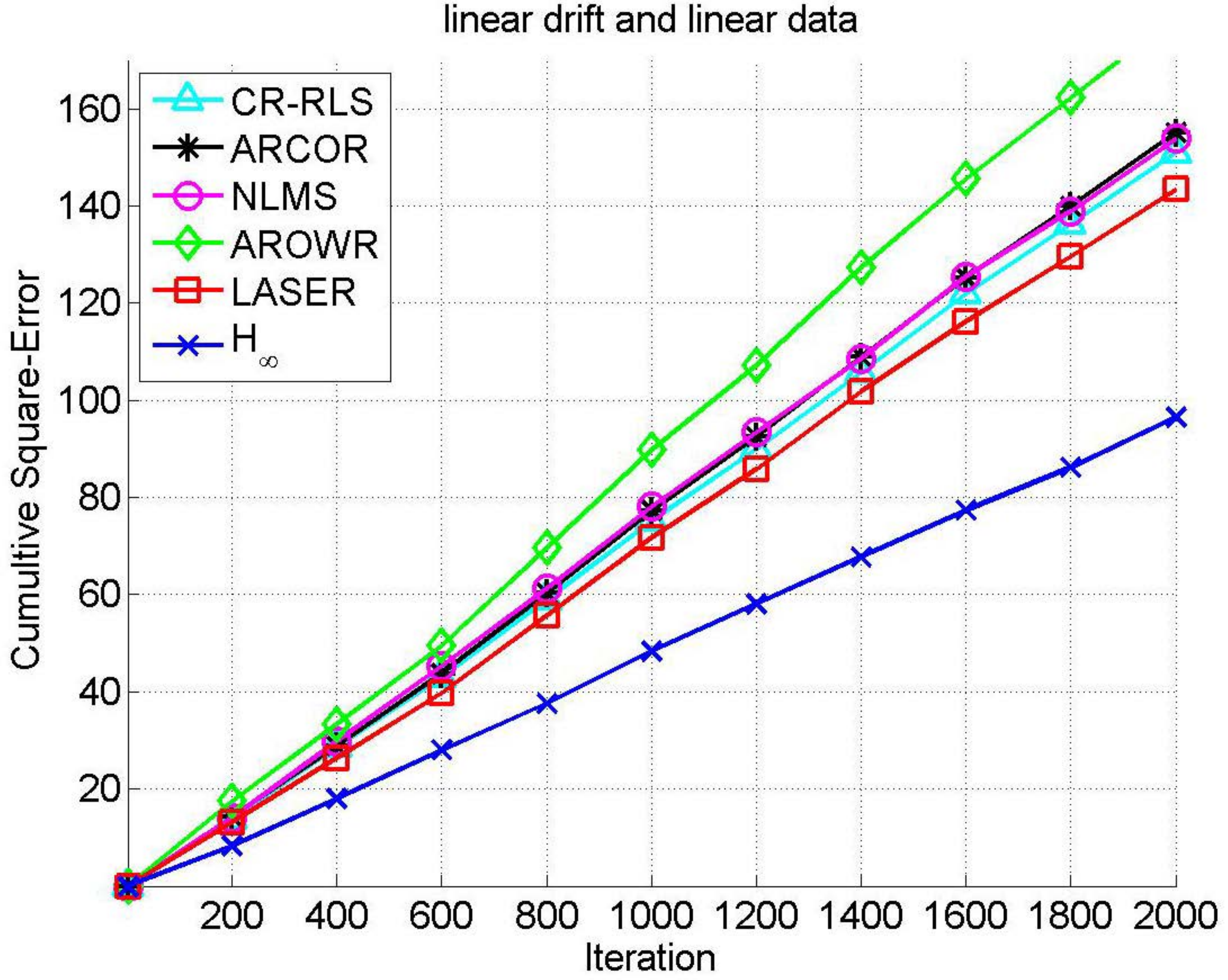}}
\!\!
\subfigure{\includegraphics[width=0.24\textwidth]{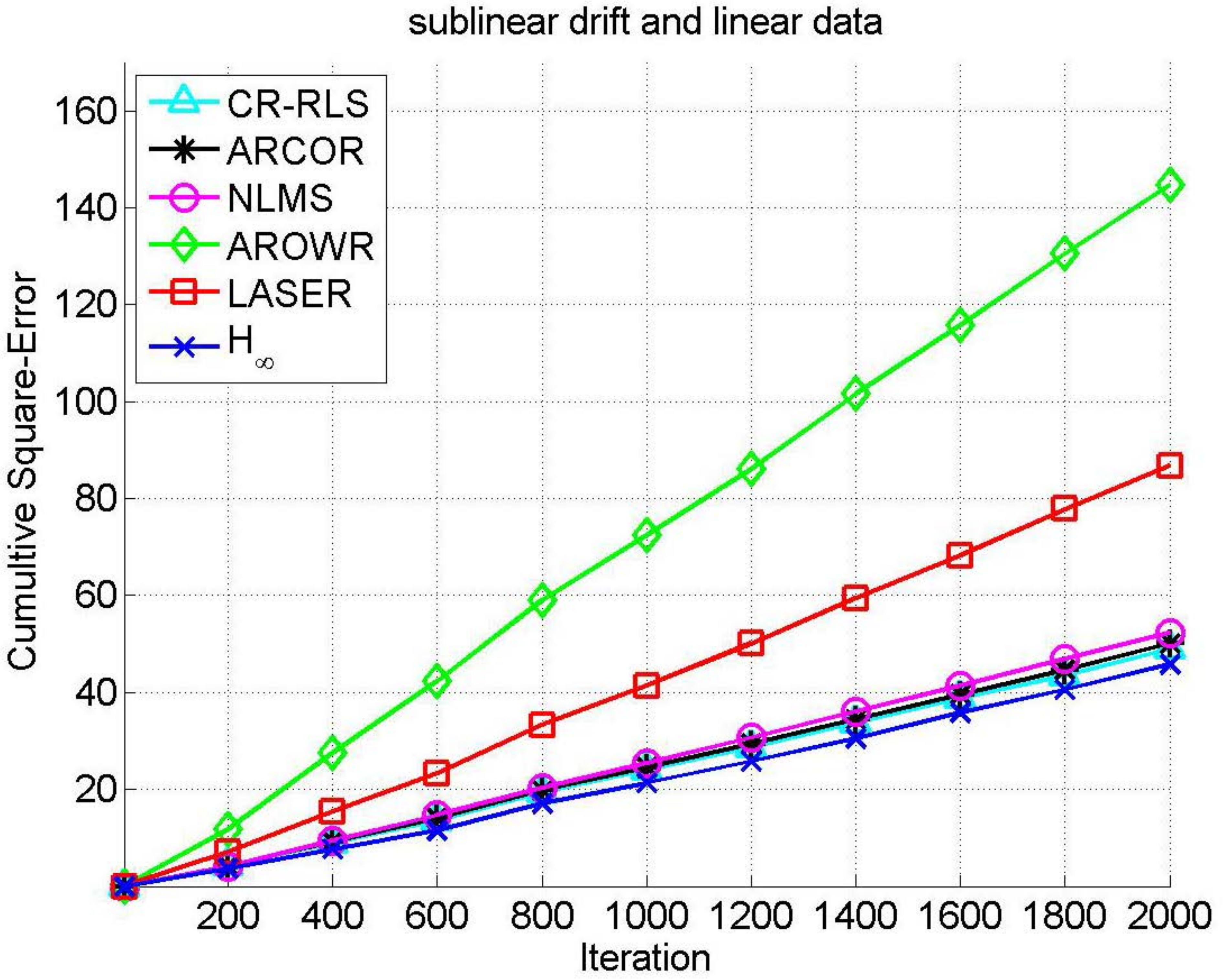}}
\subfigure{\includegraphics[width=0.24\textwidth]{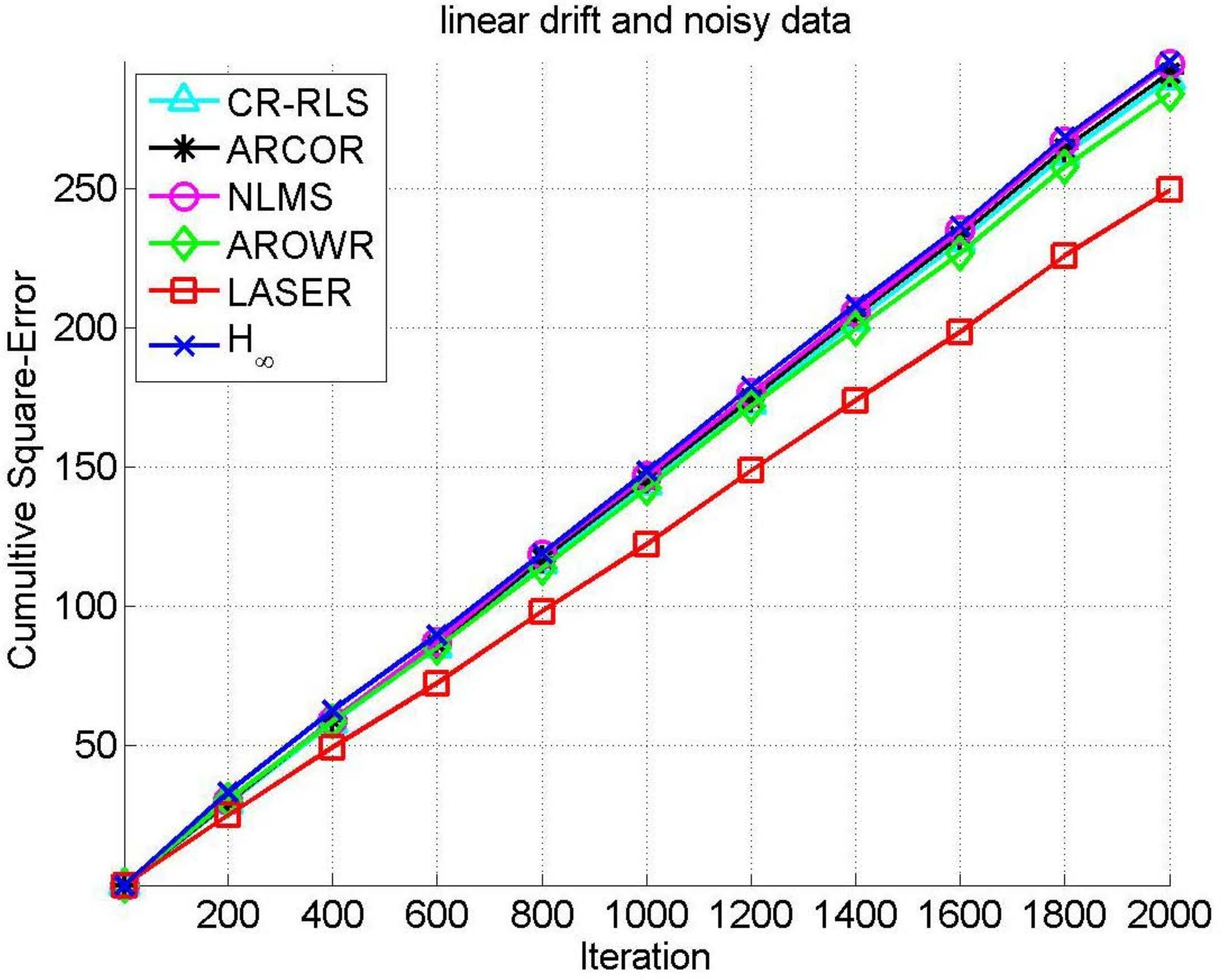}}
\!\!
\subfigure{\includegraphics[width=0.24\textwidth]{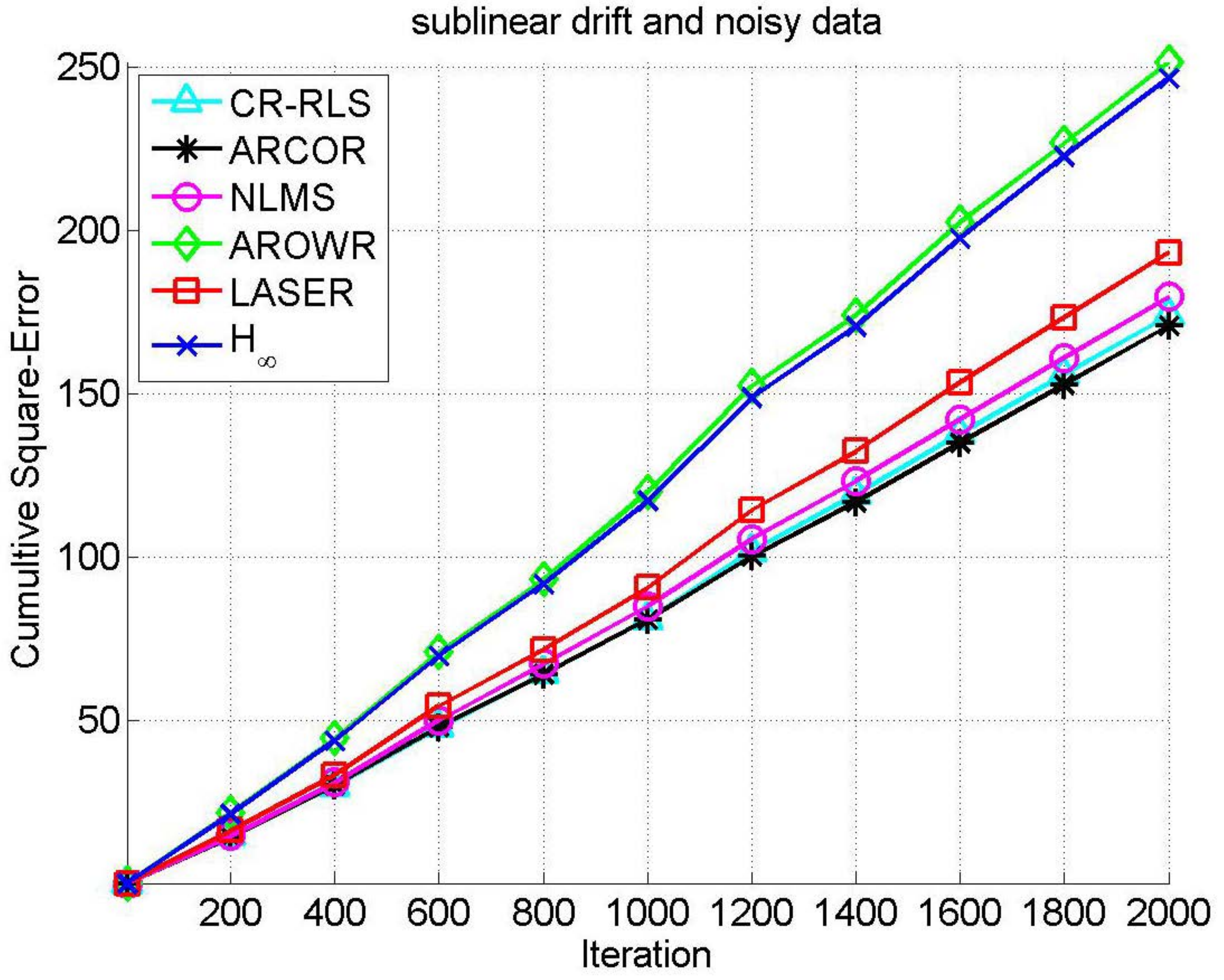}}
\caption{Cumulative squared loss for AROWR, ARCOR, NLMS, CR-RLS, LASER and $H_\infty$ vs iteration. Top left - linear drift and linear data, top right - sublinear drift and linear data, bottom left - linear drift and noisy data, bottom right - sublinear drift and noisy data.}
\label{fig:sims}
\end{figure}

We evaluate the LASER and $H_\infty$ algorithms on four synthetic datasets.
We set $T=2000$ and $d=20$. For all datasets, the
inputs $\mathbf{x}_t\in\reals^{20}$ were generated such that
the first ten coordinates were grouped into five groups of size
two. Each such pair was drawn from a $45^\circ$ rotated Gaussian
distribution with standard deviations $10$ and $1$. The remaining $10$
coordinates were drawn from independent Gaussian distributions
$\norm\paren{0,2}$. The first synthetic dataset was
generated using a sequence of vectors $\vui{t}\in\reals^{20}$ for
which the only non-zero coordinates are the first two, where their
values are the coordinates of a unit vector that is rotating with a
constant rate (linear drift). Specifically, we have $\Vert\vui{t}\Vert=1$ and the instantaneous drift $\Vert\vui{t}-\vui{t-1}\Vert$ is constant.  The
second synthetic dataset was generated using a sequence
of vectors $\vui{t}\in\reals^{20}$ for which the only non-zero
coordinates are the first two. This vector in $\reals^2$ is of unit
norm $\Vert\vui{t}\Vert=1$ and rotating in a rate of $t^{-1}$ (sublinear drift). In addition every $50$ time-steps the two-dimensional vector defined
above was ``embedded'' in different pair of coordinates of the
reference vector $\vui{t}$, for the first $50$ steps it were
coordinates $1,2$, in the next $50$ examples, coordinates $3,4$, and
so on. This change causes a switch in the reference vector $\vui{t}$. For the first two datasets we set $\yi{t}=\vxti{t}\vui{t}$ (linear data).
The third and fourth datasets are the same as first and second except we set $\yi{t}=\vxti{t}\vui{t}+n_t$ where
$n_t\sim\norm\paren{0,0.05}$ (noisy data).

We compared six algorithms: NLMS (normalized least mean
square)~\cite{Bershad,Bitmead} which is a state-of-the-art first-order
algorithm, AROWR (AROW for Regression)~\cite{CrammerKuDr09},
ARCOR~\cite{VaitsCr11}, CR-RLS~\cite{Chen,Salgado}, LASER and $H_\infty$.  The
algorithms' parameters were tuned using a single random sequence. We
repeat each experiment $100$ times reporting the mean cumulative
square-loss. The results are summarized in \figref{fig:sims} (best viewed in color). 


For the first and third datasets (left plots of
\figref{fig:sims}) we observe the superior performance of the LASER algorithm over previous approaches. LASER has a good tracking ability, fast learning rate and it is designed to perform well in severe conditions like linear drift.

For the second and fourth datasets (right plots of
\figref{fig:sims}), where we have sublinear drift level, we get that ARCOR
outperforms LASER since it is especially designed for sublinear amount
of data drift, yet, $H_\infty$ outperforms ARCOR when there is no noise (top-right plot).

For the third and fourth datasets (bottom plots of \figref{fig:sims}), where we added noise to labels, the performance of $H_\infty$ degrades, as expected from our discussion in \secref{H8_sec}.

\section{Related Work}
The problem of performing online regression was studied for more than
fifty years in statistics, signal processing and machine learning. We
already mentioned the work of Widrow and Hoff~\cite{WidrowHoff} who studied a gradient
descent algorithm for the squared loss. Many variants of the algorithm
were studied since then. A notable example is the 
normalized least mean squares algorithm (NLMS)~\cite{Bitmead,Bershad}
that adapts to the input's scale.

There exists a large body of work on this problem proposed by the machine
learning community, which clearly cannot be covered fully here. We
refer the reader to a encyclopedic book in the subject
\cite{CesaBiGa06}. Gradient descent based algorithms for regression
with the squared loss were proposed by Cesa-Bianchi
et al.~\cite{Nicolo_Warmuth} about two decades ago. These algorithms
were generalized and extended by Kivinen and Warmuth~\cite{Kiv_War} using additional
regularization functions.

An online version of the ridge regression algorithm in the worst-case
setting was proposed and analyzed by Foster~\cite{Foster91}. A related
algorithm called Aggregating Algorithm (AA) was studied by
Vovk~\cite{vovkAS}, and later applied to the problem of linear regression with square loss~\cite{Vovk01}.  The recursive least squares (RLS)~\cite{Hayes} is a
similar algorithm proposed for adaptive filtering. Both algorithms
make use of second order information, as they maintain a weight-vector
and a covariance-like positive semi-definite (PSD) matrix used to
re-weight the input. The eigenvalues of this covariance-like matrix
increase with time $t$, a property which is used to prove logarithmic
regret bounds.

The derivation of our algorithm shares similarities with the work of
Forster~\cite{Forster} and the work of Moroshko and Crammer~\cite{MoroshkoCr12}. These algorithms are motivated from the
last-step min-max predictor. While the algorithms of Forster~\cite{Forster} and Moroshko and Crammer~\cite{MoroshkoCr12} are
designed for the stationary setting, our work is primarily designed
for the non-stationary setting. Moroshko and
Crammer~\cite{MoroshkoCr12} also discussed a weak variant of the
non-stationary setting, where the complexity is measured by the total
distance from a reference vector $\bar{\mathbf{u}}$, rather than the
total distance of consecutive vectors (as in this paper), which is
more relevant to non-stationary problems. Note also that Moroshko and
Crammer~\cite{MoroshkoCr12} did not derive algorithms for the non-stationary setting, but just show a bound of the weighted min-max algorithm (designed for the stationary setting) in the weak non-stationary setting.

Our work is mostly close to a recent
algorithm~\cite{VaitsCr11} called ARCOR. This algorithm
is based on the RLS algorithm with an additional projection step, and
it controls the eigenvalues of a covariance-like matrix using scheduled resets.  The Covariance Reset RLS algorithm
(CR-RLS)~\cite{Chen,Salgado,Goodhart} 
is another example of an
algorithm that resets a covariance matrix but every fixed amount of
data points, as opposed to ARCOR that performs these resets adaptively.  All
of these algorithms 
that were designed to have numerically stable computations, perform
covariance reset from time to time. Our algorithm, LASER, is simpler
as it does not involve these steps, and it controls the increase of the
eigenvalues of the covariance matrix $\mathbf{D}$ implicitly rather
than explicitly by ``averaging'' it with a fixed diagonal matrix (see
\eqref{D}). The Kalman filter ~\cite{Kalman60} 
and the $H_\infty$
algorithm~(e.g. \cite{Simon:2006:OSE:1146304}) 
designed for filtering take a similar approach, yet the exact algebraic
form is different (\figref{algorithm:laser}
vs. \figref{algorithm:hi}).

ARCOR also controls explicitly the norm of the weight vector, which is
used for its analysis, by projecting it into a bounded set, as was
also proposed by Herbster and Warmuth~\cite{HerbsterW01}. Other approaches to control its
norm are to shrink it multiplicatively~\cite{KivinenSW01} or
by removing old
examples~\cite{CavallantiCG07}. Some of these algorithms were
designed to have sparse functions in the kernel space~(e.g. 
\cite{CrammerKS03,Dekel05theforgetron}). Note that our algorithm
LASER is simpler as it does not
perform any of these operation explicitly.
Finally, few algorithms that employ second order information
were recently proposed for
classification~\cite{CesaBianchiCoGe05,CrammerKuDr09,Crammer:2012:CLC:2343676.2343704}, 
and later in the online convex programming framework
\cite{DuchiHS10,McMahanS10}.

\section{Summary and Conclusions}
We proposed a novel algorithm for non-stationary online regression designed and
analyzed with the squared loss. The algorithm was developed from the
last-step minmax predictor for {\em non-stationary} problems, and we
showed an exact recursive form of its solution. We also described an
algorithm based on the $H_\infty$ filter, that is motivated from a
min-max approach as well, yet for filtering, and bounded its
regret. Simulations showed its superior performance in
a worst-case (close to a constant per iteration) drift.

An interesting future direction is to extend
the algorithm for general loss functions rather than the squared
loss. Currently, to implement the algorithm we need to perform either
matrix inversion or eigenvector decomposition, we like to design
a more efficient version of the algorithm. Additionally, for the
algorithm to perform well, the amount of drift $V$ or a bound over it
are used by the algorithm. An interesting direction is to design
algorithms that automatically detect the level of drift, or are
invariant to it.
\appendix
\section{Proofs}
\vspace{-0.25cm}
\subsection{Proof of \corref{cor:main}}
\vspace{-0.25cm}
\label{proof_cor_main}
\begin{proof}
Plugging \lemref{lem:bound_1} in \thmref{thm:basic_bound} we have for
all $(\vui{1} \dots \vui{T})$,
\begin{align*}
L_T(\textrm{LASER})
 &\leq  b\left\Vert
    \mathbf{u}_{1}\right\Vert ^{2}+c V+L_T(\{\vui{t}\})\\
&+Y^{2} \ln\left|\frac{1}{b}\mathbf{D}_{T}\right|+c^{-1}Y^2 \sum_{t=1}^{T} \tr\paren{\mathbf{D}_{t-1}} ~.
\end{align*}
Using \lemref{eigen_values_lemma} we bound the RHS and get
\begin{align}
L_T(\textrm{LASER})\leq 
  b\left\Vert
     \mathbf{u}_{1}\right\Vert ^{2}+L_T(\{\vui{t}\})+Y^{2} \ln\left|\frac{1}{b}\mathbf{D}_{T}\right|\nonumber\\ +c^{-1}Y^2\tr\paren{\mathbf{D}_0}+c V\nonumber\\
 +c^{-1}Y^2 T d  \max\braces{ \frac{3X^2 +
   \sqrt{X^4+4X^2 c}}{2},b+X^2}~.\nonumber
\end{align} 
%
The term $c^{-1}Y^2\tr\paren{\mathbf{D}_0}$ does not depend on $T$, because
\(
c^{-1}Y^2\tr\paren{\mathbf{D}_0}=c^{-1}Y^2d\frac{bc}{c-b}=\frac{\varepsilon}{1-\varepsilon}Y^{2}d ~.
\)
To show \eqref{bound1}, note that 
\(
V \leq T \frac{\sqrt{2}Y^2dX}{\mu^{3/2}} \Leftrightarrow \mu \leq  \paren{\frac{\sqrt{2}Y^2dXT}{V}}^{2/3} 
=c~.
\)
We thus have that $\paren{ 3X^2 +
   \sqrt{ X^4+4X^2 c }}/ {2 } \leq \paren{ 3X^2 +
   \sqrt{ 8X^2 c }}/{ 2 } \leq \sqrt{ 8X^2 c }$, and we get a bound on
 the right term of \eqref{final_cor},
\begin{align*}
 \max\braces{  \paren{ 3X^2 +
    \sqrt{ X^4+4X^2 c } }/{ 2 },b+X^2 } \leq \\
 \max\braces{  \sqrt{ 8X^2 c },b+X^2 }
\leq 2X\sqrt{ 2c } ~.
\end{align*}
%
Using this bound and plugging the value of $c$ from \eqref{c1} we bound
\eqref{final_cor} and conclude the proof, 
\begin{align*}
\paren{\frac{\sqrt{2}T Y^2 d X}{V}}^{2/3} V 
&+ Y^2 T d 2X
\sqrt{2 \paren{\frac{\sqrt{2}T Y^2 d X}{V}}^{-2/3}}\\
 &= 
3\paren{\sqrt{2}T Y^2 d X}^{2/3} V^{1/3} ~.
\end{align*}
\end{proof}



{
\bibliographystyle{abbrv}
\bibliography{bib}
}

\section{APPENDIX\\SUPPLEMENTARY MATERIAL}
\label{sec:supp_material}

\subsection{Proof of \lemref{lem:lemma11}}
\label{proof_lemma11}

\begin{proof}
We calculate
\begin{align*}
P_{t}\left(\mathbf{u}_{t}\right) 
=&  \min_{\mathbf{u}_{1},\ldots,\mathbf{u}_{t-1}} \Bigg(b\left\Vert \mathbf{u}_{1}\right\Vert ^{2}+c\sum_{s=1}^{t-1}\left\Vert \mathbf{u}_{s+1}-\mathbf{u}_{s}\right\Vert ^{2}\\
&+\sum_{s=1}^{t}\left(y_{s}-\mathbf{u}_{s}^{\top}\mathbf{x}_{s}\right)^{2}\Bigg)\\
  =&
  \min_{\mathbf{u}_{t-1}}\min_{\mathbf{u}_{1},\ldots,\mathbf{u}_{t-2}}\Bigg(b\left\Vert
    \mathbf{u}_{1}\right\Vert ^{2}+c\sum_{s=1}^{t-2}\left\Vert
    \mathbf{u}_{s+1}-\mathbf{u}_{s}\right\Vert
  ^{2}\\
&+\sum_{s=1}^{t-1}\left(y_{s}-\mathbf{u}_{s}^{\top}\mathbf{x}_{s}\right)^{2}
+c\left\Vert \mathbf{u}_{t}-\mathbf{u}_{t-1}\right\Vert ^{2}\\
&+\left(y_{t}-\mathbf{u}_{t}^{\top}\mathbf{x}_{t}\right)^{2}\Bigg)\\
   =&  \min_{\mathbf{u}_{t-1}}\Bigg(P_{t-1}\left(\mathbf{u}_{t-1}\right)+c\left\Vert \mathbf{u}_{t}-\mathbf{u}_{t-1}\right\Vert ^{2}\\
&+\left(y_{t}-\mathbf{u}_{t}^{\top}\mathbf{x}_{t}\right)^{2}\Bigg)
\end{align*}
\end{proof}

\subsection{Proof of \lemref{lem:lemma12}}
\label{proof_lemma12}
\begin{proof}
By definition, 
\(
P_{1}\left(\mathbf{u}_{1}\right)  =  Q_{1}\left(\mathbf{u}_{1}\right)
  =  b\left\Vert \mathbf{u}_{1}\right\Vert ^{2}+\left(y_{1}-\mathbf{u}_{1}^{\top}\mathbf{x}_{1}\right)^{2}
  = 
 \mathbf{u}_{1}^{\top}\left(b\mathbf{I}+\mathbf{x}_{1}\mathbf{x}_{1}^{\top}\right)\mathbf{u}_{1}-2y_{1}\mathbf{u}_{1}^{\top}\mathbf{x}_{1}+y_{1}^{2} ~,
\)
and indeed
\(
\mathbf{D}_{1}=b\mathbf{I}+\mathbf{x}_{1}\mathbf{x}_{1}^{\top}
\), 
\(
\mathbf{e}_{1}=y_{1}\mathbf{x}_{1}
\), and
\(
f_{1}=y_{1}^{2}
\).
We proceed by induction, assume that, 
\(
P_{t-1}\left(\mathbf{u}_{t-1}\right)=\mathbf{u}_{t-1}^{\top}\mathbf{D}_{t-1}\mathbf{u}_{t-1}-2\mathbf{u}_{t-1}^{\top}\mathbf{e}_{t-1}+f_{t-1}
\).
Applying \lemref{lem:lemma11} we get,
\begin{align*}
P_{t}\left(\mathbf{u}_{t}\right) 
=&  \min_{\mathbf{u}_{t-1}}\Bigg(\mathbf{u}_{t-1}^{\top}\mathbf{D}_{t-1}\mathbf{u}_{t-1}-2\mathbf{u}_{t-1}^{\top}\mathbf{e}_{t-1}+f_{t-1}\\
&+c\left\Vert \mathbf{u}_{t}-\mathbf{u}_{t-1}\right\Vert ^{2}+\left(y_{t}-\mathbf{u}_{t}^{\top}\mathbf{x}_{t}\right)^{2}\Bigg)\\
 =&  \min_{\mathbf{u}_{t-1}}\Bigg(\mathbf{u}_{t-1}^{\top}\left(c\mathbf{I}+\mathbf{D}_{t-1}\right)\mathbf{u}_{t-1}\\
&-2\mathbf{u}_{t-1}^{\top}\left(c\mathbf{u}_{t}+\mathbf{e}_{t-1}\right)+f_{t-1}+c\left\Vert \mathbf{u}_{t}\right\Vert ^{2}\\
&+\left(y_{t}-\mathbf{u}_{t}^{\top}\mathbf{x}_{t}\right)^{2}\Bigg)\\
 =&  -\left(c\mathbf{u}_{t}+\mathbf{e}_{t-1}\right)^{\top}\left(c\mathbf{I}+\mathbf{D}_{t-1}\right)^{-1}\left(c\mathbf{u}_{t}+\mathbf{e}_{t-1}\right)\\
&+f_{t-1}+c\left\Vert \mathbf{u}_{t}\right\Vert ^{2}+\left(y_{t}-\mathbf{u}_{t}^{\top}\mathbf{x}_{t}\right)^{2}\\
 =&  \mathbf{u}_{t}^{\top}\left(c\mathbf{I}+\mathbf{x}_{t}\mathbf{x}_{t}^{\top}-c^{2}\left(c\mathbf{I}+\mathbf{D}_{t-1}\right)^{-1}\right)\mathbf{u}_{t}\\
 &  
 -2\mathbf{u}_{t}^{\top}\left[c\left(c\mathbf{I}+\mathbf{D}_{t-1}\right)^{-1}\mathbf{e}_{t-1}+y_{t}\mathbf{x}_{t}\right]\\
&-\mathbf{e}_{t-1}^{\top}\left(c\mathbf{I}+\mathbf{D}_{t-1}\right)^{-1}\mathbf{e}_{t-1}+f_{t-1}+y_{t}^{2}
\end{align*}
Using Woodbury identity we continue to develop the last equation,
\begin{align*}
 = &\mathbf{u}_{t}^{\top}\left(c\mathbf{I}+\mathbf{x}_{t}\mathbf{x}_{t}^{\top}\right.\\
&\left.-c^{2}\left[c^{-1}\mathbf{I}-c^{-2}\left(\mathbf{D}_{t-1}^{-1}+c^{-1}\mathbf{I}\right)^{-1}\right]\right)\mathbf{u}_{t}\\
 &  -2\mathbf{u}_{t}^{\top}\left[\left(\mathbf{I}+c^{-1}\mathbf{D}_{t-1}\right)^{-1}\mathbf{e}_{t-1}+y_{t}\mathbf{x}_{t}\right]\\
&-\mathbf{e}_{t-1}^{\top}\left(c\mathbf{I}+\mathbf{D}_{t-1}\right)^{-1}\mathbf{e}_{t-1}+f_{t-1}+y_{t}^{2}\\
  = & \mathbf{u}_{t}^{\top}\left(\left(\mathbf{D}_{t-1}^{-1}+c^{-1}\mathbf{I}\right)^{-1}+\mathbf{x}_{t}\mathbf{x}_{t}^{\top}\right)\mathbf{u}_{t}\\
 &   -2\mathbf{u}_{t}^{\top}\left[\left(\mathbf{I}+c^{-1}\mathbf{D}_{t-1}\right)^{-1}\mathbf{e}_{t-1}+y_{t}\mathbf{x}_{t}\right]\\
&-\mathbf{e}_{t-1}^{\top}\left(c\mathbf{I}+\mathbf{D}_{t-1}\right)^{-1}\mathbf{e}_{t-1}+f_{t-1}+y_{t}^{2}~,
\end{align*}
and indeed
\(
\mathbf{D}_{t}=\left(\mathbf{D}_{t-1}^{-1}+c^{-1}\mathbf{I}\right)^{-1}+\mathbf{x}_{t}\mathbf{x}_{t}^{\top}
\), \\
\(
\mathbf{e}_{t}=\left(\mathbf{I}+c^{-1}\mathbf{D}_{t-1}\right)^{-1}\mathbf{e}_{t-1}+y_{t}\mathbf{x}_{t}
\) and,
\(
f_{t}=f_{t-1}-\mathbf{e}_{t-1}^{\top}\left(c\mathbf{I}+\mathbf{D}_{t-1}\right)^{-1}\mathbf{e}_{t-1}+y_{t}^{2}
\), as desired.
\end{proof}

\subsection{Proof of \lemref{lem:technical}}
\label{proof_lemma_technical}
\begin{proof}
We first use  the Woodbury equation to get the following two identities
\begin{align*}
\mathbf{D}_{t}^{-1}&=\left[\left(\mathbf{D}_{t-1}^{-1}+c^{-1}\mathbf{I}\right)^{-1}+\mathbf{x}_{t}\mathbf{x}_{t}^{\top}\right]^{-1}\\
&=\mathbf{D}_{t-1}^{-1}+c^{-1}\mathbf{I}\\
&-\frac{\left(\mathbf{D}_{t-1}^{-1}+c^{-1}\mathbf{I}\right)\mathbf{x}_{t}\mathbf{x}_{t}^{\top}\left(\mathbf{D}_{t-1}^{-1}+c^{-1}\mathbf{I}\right)}{1+\mathbf{x}_{t}^{\top}\left(\mathbf{D}_{t-1}^{-1}+c^{-1}\mathbf{I}\right)\mathbf{x}_{t}}\\
\end{align*}
and
\begin{align*}
\left(\mathbf{I}+c^{-1}\mathbf{D}_{t-1}\right)^{-1}&=\mathbf{I}-c^{-1}\left(\mathbf{D}_{t-1}^{-1}+c^{-1}\mathbf{I}\right)^{-1}
\end{align*}
Multiplying both identities with each other we get,
\begin{align}
 & \mathbf{D}_{t}^{-1}\left(\mathbf{I}+c^{-1}\mathbf{D}_{t-1}\right)^{-1}\nonumber\\
 =~&  \Bigg[\mathbf{D}_{t-1}^{-1}+c^{-1}\mathbf{I}\nonumber\\
&-\frac{\left(\mathbf{D}_{t-1}^{-1}+c^{-1}\mathbf{I}\right)\mathbf{x}_{t}\mathbf{x}_{t}^{\top}\left(\mathbf{D}_{t-1}^{-1}+c^{-1}\mathbf{I}\right)}{1+\mathbf{x}_{t}^{\top}\left(\mathbf{D}_{t-1}^{-1}+c^{-1}\mathbf{I}\right)\mathbf{x}_{t}}\Bigg]\Bigg[\mathbf{I}\nonumber\\
&-c^{-1}\left(\mathbf{D}_{t-1}^{-1}+c^{-1}\mathbf{I}\right)^{-1}\Bigg]\nonumber\\
 =~ & \mathbf{D}_{t-1}^{-1}-\frac{\left(\mathbf{D}_{t-1}^{-1}+c^{-1}\mathbf{I}\right)\mathbf{x}_{t}\mathbf{x}_{t}^{\top}\mathbf{D}_{t-1}^{-1}}{1+\mathbf{x}_{t}^{\top}\left(\mathbf{D}_{t-1}^{-1}+c^{-1}\mathbf{I}\right)\mathbf{x}_{t}}\label{identity1}
\end{align}
and, similarly, we multiply the identities in the other order and get,
\begin{align}
 & \left(\mathbf{I}+c^{-1}\mathbf{D}_{t-1}\right)^{-1}\mathbf{D}_{t}^{-1}
 \nonumber\\
=~& \mathbf{D}_{t-1}^{-1}-\frac{\mathbf{D}_{t-1}^{-1}\mathbf{x}_{t}\mathbf{x}_{t}^{\top}\left(\mathbf{D}_{t-1}^{-1}+c^{-1}\mathbf{I}\right)}{1+\mathbf{x}_{t}^{\top}\left(\mathbf{D}_{t-1}^{-1}+c^{-1}\mathbf{I}\right)\mathbf{x}_{t}}\label{identity2}
\end{align}

Finally, from \eqref{identity1} we get,
\begin{align*}
 &\left(\mathbf{I}+c^{-1}\mathbf{D}_{t-1}\right)^{-1}\mathbf{D}_{t}^{-1}\mathbf{x}_{t}\mathbf{x}_{t}^{\top}\mathbf{D}_{t}^{-1}\left(\mathbf{I}+c^{-1}\mathbf{D}_{t-1}\right)^{-1}\\
&-\mathbf{D}_{t-1}^{-1}\\
&+\left(\mathbf{I}+c^{-1}\mathbf{D}_{t-1}\right)^{-1}\left[\mathbf{D}_{t}^{-1}\left(\mathbf{I}+c^{-1}\mathbf{D}_{t-1}\right)^{-1}\right.\\
&\left.+c^{-1}\mathbf{I}\right]\\
 =~&
 \left(\mathbf{I}+c^{-1}\mathbf{D}_{t-1}\right)^{-1}\mathbf{D}_{t}^{-1}\mathbf{x}_{t}\mathbf{x}_{t}^{\top}\mathbf{D}_{t}^{-1}\left(\mathbf{I}+c^{-1}\mathbf{D}_{t-1}\right)^{-1}\\
&-\mathbf{D}_{t-1}^{-1}\\
&+\left[\mathbf{I}-c^{-1}\left(\mathbf{D}_{t-1}^{-1}+c^{-1}\mathbf{I}\right)^{-1}\right]\Bigg[\mathbf{D}_{t-1}^{-1}+c^{-1}\mathbf{I}\\
&-\frac{\left(\mathbf{D}_{t-1}^{-1}+c^{-1}\mathbf{I}\right)\mathbf{x}_{t}\mathbf{x}_{t}^{\top}\mathbf{D}_{t-1}^{-1}}{1+\mathbf{x}_{t}^{\top}\left(\mathbf{D}_{t-1}^{-1}+c^{-1}\mathbf{I}\right)\mathbf{x}_{t}}\Bigg]
\end{align*}
We 
develop the last equality and use \eqref{identity1} and
\eqref{identity2} in the second equality below,
\begin{align*}
 =~ &
 \left(\mathbf{I}+c^{-1}\mathbf{D}_{t-1}\right)^{-1}\mathbf{D}_{t}^{-1}\mathbf{x}_{t}\mathbf{x}_{t}^{\top}\mathbf{D}_{t}^{-1}\left(\mathbf{I}+c^{-1}\mathbf{D}_{t-1}\right)^{-1}\\
&-\mathbf{D}_{t-1}^{-1}+\mathbf{D}_{t-1}^{-1}-\frac{\mathbf{D}_{t-1}^{-1}\mathbf{x}_{t}\mathbf{x}_{t}^{\top}\mathbf{D}_{t-1}^{-1}}{1+\mathbf{x}_{t}^{\top}\left(\mathbf{D}_{t-1}^{-1}+c^{-1}\mathbf{I}\right)\mathbf{x}_{t}}\\
 =~ &
 \left[\mathbf{D}_{t-1}^{-1}-\frac{\mathbf{D}_{t-1}^{-1}\mathbf{x}_{t}\mathbf{x}_{t}^{\top}\left(\mathbf{D}_{t-1}^{-1}+c^{-1}\mathbf{I}\right)}{1+\mathbf{x}_{t}^{\top}\left(\mathbf{D}_{t-1}^{-1}+c^{-1}\mathbf{I}\right)\mathbf{x}_{t}}\right]\mathbf{x}_{t}\mathbf{x}_{t}^{\top}\\
&\left[\mathbf{D}_{t-1}^{-1}-\frac{\left(\mathbf{D}_{t-1}^{-1}+c^{-1}\mathbf{I}\right)\mathbf{x}_{t}\mathbf{x}_{t}^{\top}\mathbf{D}_{t-1}^{-1}}{1+\mathbf{x}_{t}^{\top}\left(\mathbf{D}_{t-1}^{-1}+c^{-1}\mathbf{I}\right)\mathbf{x}_{t}}\right]\\
&-\frac{\mathbf{D}_{t-1}^{-1}\mathbf{x}_{t}\mathbf{x}_{t}^{\top}\mathbf{D}_{t-1}^{-1}}{1+\mathbf{x}_{t}^{\top}\left(\mathbf{D}_{t-1}^{-1}+c^{-1}\mathbf{I}\right)\mathbf{x}_{t}}
 \\
=~& 
-\frac{\mathbf{x}_{t}^{\top}\left(\mathbf{D}_{t-1}^{-1}+c^{-1}\mathbf{I}\right)\mathbf{x}_{t}\mathbf{D}_{t-1}^{-1}\mathbf{x}_{t}\mathbf{x}_{t}^{\top}\mathbf{D}_{t-1}^{-1}}{\left(1+\mathbf{x}_{t}^{\top}\left(\mathbf{D}_{t-1}^{-1}+c^{-1}\mathbf{I}\right)\mathbf{x}_{t}\right)^{2}}~~\preceq~~0
\end{align*}
\end{proof}

\subsection{Derivations for \thmref{thm:basic_bound}}
\label{algebraic_manipulation}
\begin{align*}
 &  \left(y_{t}-\hat{y}_{t}\right)^{2}+\min_{\mathbf{u}_{1},\ldots,\mathbf{u}_{t-1}}Q_{t-1}\left(\mathbf{u}_{1},\ldots,\mathbf{u}_{t-1}\right)\\
&-\min_{\mathbf{u}_{1},\ldots,\mathbf{u}_{t}}Q_{t}\left(\mathbf{u}_{1},\ldots,\mathbf{u}_{t}\right)\\
 =&  \left(y_{t}-\hat{y}_{t}\right)^{2}-\mathbf{e}_{t-1}^{\top}\mathbf{D}_{t-1}^{-1}\mathbf{e}_{t-1}+f_{t-1}+\mathbf{e}_{t}^{\top}\mathbf{D}_{t}^{-1}\mathbf{e}_{t}-f_{t}\\
 =&  \left(y_{t}-\hat{y}_{t}\right)^{2}-\mathbf{e}_{t-1}^{\top}\mathbf{D}_{t-1}^{-1}\mathbf{e}_{t-1}\\
&+\mathbf{e}_{t-1}^{\top}\left(c\mathbf{I}+\mathbf{D}_{t-1}\right)^{-1}\mathbf{e}_{t-1}-y_{t}^{2}\\
 &  
 +\left(\left(\mathbf{I}+c^{-1}\mathbf{D}_{t-1}\right)^{-1}\mathbf{e}_{t-1}+y_{t}\mathbf{x}_{t}\right)^{\top}\mathbf{D}_{t}^{-1}\\
&\left(\left(\mathbf{I}+c^{-1}\mathbf{D}_{t-1}\right)^{-1}\mathbf{e}_{t-1}+y_{t}\mathbf{x}_{t}\right) 
\end{align*}
where the last equality follows \eqref{e}. We proceed to develop the
last equality,
\begin{align*}
=&  \left(y_{t}-\hat{y}_{t}\right)^{2}-\mathbf{e}_{t-1}^{\top}\mathbf{D}_{t-1}^{-1}\mathbf{e}_{t-1}\\
&+\mathbf{e}_{t-1}^{\top}\left(c\mathbf{I}+\mathbf{D}_{t-1}\right)^{-1}\mathbf{e}_{t-1}-y_{t}^{2}\\
& +\mathbf{e}_{t-1}^{\top}\left(\mathbf{I}+c^{-1}\mathbf{D}_{t-1}\right)^{-1}\mathbf{D}_{t}^{-1}\left(\mathbf{I}+c^{-1}\mathbf{D}_{t-1}\right)^{-1}\mathbf{e}_{t-1}\\
&+2y_{t}\mathbf{x}_{t}^{\top}\mathbf{D}_{t}^{-1}\left(\mathbf{I}+c^{-1}\mathbf{D}_{t-1}\right)^{-1}\mathbf{e}_{t-1}+y_{t}^{2}\mathbf{x}_{t}^{\top}\mathbf{D}_{t}^{-1}\mathbf{x}_{t}\\
 =&   \left(y_{t}-\hat{y}_{t}\right)^{2}+\mathbf{e}_{t-1}^{\top}\Bigg(-\mathbf{D}_{t-1}^{-1}+\\
&\left(\mathbf{I}+c^{-1}\mathbf{D}_{t-1}\right)^{-1}\left[\mathbf{D}_{t}^{-1}\left(\mathbf{I}+c^{-1}\mathbf{D}_{t-1}\right)^{-1}\right.\\
&\left.+c^{-1}\mathbf{I}\right]\Bigg)\mathbf{e}_{t-1}+2y_{t}\mathbf{x}_{t}^{\top}\mathbf{D}_{t}^{-1}\left(\mathbf{I}+c^{-1}\mathbf{D}_{t-1}\right)^{-1}\mathbf{e}_{t-1}\\
&+y_{t}^{2}\mathbf{x}_{t}^{\top}\mathbf{D}_{t}^{-1}\mathbf{x}_{t}-y_{t}^{2}~.
\end{align*}

\subsection{Details for the bound \eqref{high_drift}}
\label{details_for_second_bound}
To show the bound \eqref{high_drift}, note that,
\(
V \geq T \frac{Y^2dM}{\mu^{2}} \Leftrightarrow \mu \geq \sqrt{ \frac{TY^2dM}{V}}=c~.
\)
We thus have that the right term of \eqref{final_cor} is upper bounded
as follows,
\begin{align*}
&\max\braces{ \frac{3X^2 +
   \sqrt{X^4+4X^2 c}}{2},b+X^2} \\
\leq& 
\max\braces{ 3X^2,\sqrt{X^4+4X^2 c},b+X^2} \\
\leq& 
\max\braces{ 3X^2,\sqrt{2}X^2 ,\sqrt{8X^2 c},b+X^2}\\ 
 =& \sqrt{8X^2}
 \max\braces{ \frac{3X^2}{\sqrt{8X^2}},\sqrt{
     c},\frac{b+X^2}{\sqrt{8X^2}}} \\
 =& \sqrt{8X^2}
 \sqrt{\max\braces{ \frac{(3X^2)^2}{{8X^2}},
     c,\frac{\paren{b+X^2}^2}{{8X^2}}}} \\
=& \sqrt{8X^2}
\sqrt{\max\braces{ \mu,c}}
\leq \sqrt{8X^2} \sqrt{\mu} = M ~.
\end{align*}
Using this bound and plugging $c=
 \sqrt{{Y^2dMT}/{V}}$ we bound
\eqref{final_cor}, 
\(
\sqrt{\frac{Y^2dMT}{V}}V + \frac{1}{\sqrt{\frac{Y^2dMT}{V}}} TdY^2M =
2\sqrt{Y^2dMTV} ~.
\)

\subsection{Proof of \lemref{operator_scalar}}
\label{proof_operator_scalar}
\begin{proof}
For the first property of the lemma we have that
 $f(\lambda) = {\lambda \beta}/\paren{\lambda+ \beta} + x^2 \leq \beta\times 1 + x^2$.
The second property follows from the symmetry between $\beta$ and
$\lambda$.
To prove the third property we decompose the function as, 
\(
f(\lambda) = \lambda - \frac{\lambda^2 }{\lambda+ \beta} + x^2 
\). 
Therefore, the function is bounded by its argument $f(\lambda)\leq
\lambda$ if, and only if, $- \frac{\lambda^2 }{\lambda+ \beta} + x^2
\leq 0$.
Since we assume $x^2\leq\gamma^2$, the last inequality holds if,
\(
-\lambda^2 + \gamma^2 \lambda + \gamma^2\beta \leq 0
\),
which holds for $\lambda \geq \frac{\gamma^2 +
  \sqrt{\gamma^4+4\gamma^2\beta}}{2}$. 

To conclude. If $\lambda \geq \frac{\gamma^2 +
  \sqrt{\gamma^4+4\gamma^2\beta}}{2}$, then $f(\lambda) \leq \lambda$. 
Otherwise, by the second property, we have, $f(\lambda) \leq \lambda+\gamma^2
\leq \frac{\gamma^2 +
  \sqrt{\gamma^4+4\gamma^2\beta}}{2} + \gamma^2 = \frac{3\gamma^2 +
  \sqrt{\gamma^4+4\gamma^2\beta}}{2}$, as required.
\end{proof}

\end{document}